\documentclass[]{article}

\usepackage{proceed2e}
\usepackage{times}
\usepackage{amsmath}
\usepackage{amsfonts}
\usepackage{graphicx}
\usepackage{amsthm}
\usepackage{dsfont}
\usepackage[round]{natbib}
\usepackage{hyperref}
\usepackage{algorithmic}
\usepackage{algorithm}

\DeclareMathOperator{\E}{\mathbb{E}}

\newcommand{\h}[1]{\widehat{#1}}
\newcommand{\Rad}{\mathfrak{R}}
\newcommand{\wt}{\widetilde}
\newcommand{\e}{\epsilon}
\newcommand{\Rset}{\mathbb{R}}
\newcommand{\mat}[1]{\mathbf{#1}}
\newcommand{\Ind}{\mathds{1}}

\newcommand{\M}{\mat{M}}

\newcommand{\q}{\widetilde{q}}
\newcommand{\coeff}{\binom{N\!-\!1}{s\!-\!1}}

\newcommand{\bbeta}{\boldsymbol{\beta}}
\newcommand{\bpsi}{\boldsymbol{\psi}}
\newcommand{\F}{\mat F}
\newcommand{\FF}{\mathsf F}
\newcommand{\G}{\mat G}
\newcommand{\GG}{\mathsf G}
\renewcommand{\u}{\mat{u}}
\renewcommand{\v}{\mat{v}}

\newcommand{\ignore}[1]{}

\newtheorem{theorem}{Theorem}
\newtheorem{lemma}{Lemma}
\newtheorem{assumption}{Assumption}
\newtheorem{corollary}{Corollary}
\newtheorem{definition}{Definition}
\newtheorem{proposition}{Proposition}

\hypersetup{colorlinks,
            linkcolor=blue,
            citecolor=blue,
            urlcolor=magenta,
            linktocpage,
            plainpages=false
}

\ignore{
\hypersetup{
  colorlinks   = true,
  urlcolor     = blue,
  linkcolor    = blue,
  citecolor    = blue
}
}

\makeatletter
\newtheorem*{rep@theorem}{\rep@title}
\newcommand{\newreptheorem}[2]{%
\newenvironment{rep#1}[1]{%
 \def\rep@title{#2 \ref{##1}}%
 \begin{rep@theorem}}%
 {\end{rep@theorem}}}
\makeatother

\makeatletter
\renewcommand{\@listI}{%
\leftmargin=20pt
\rightmargin=0pt
\labelsep=5pt
\labelwidth=20pt
\itemindent=0pt
\listparindent=0pt
\topsep=0pt plus 3pt
\partopsep=0pt plus 3pt 
\parsep=0pt plus 1pt
\itemsep=\parsep}
\makeatother

\newreptheorem{proposition}{Proposition}
\newreptheorem{theorem}{Theorem}
\begin{document}

\title{Non-parametric Revenue Optimization for Generalized Second Price Auctions}

 \author{
{\bf Mehryar Mohri} \\
Courant Institute and Google Research, \\
251 Mercer Street, New York, NY \\
\And
{\bf Andr\'es Mu\~noz Medina} \\
Courant Institute of Mathematical Sciences, \\
251 Mercer Street, New York, NY
}

\maketitle

\begin{abstract}

  We present an extensive analysis of the key problem of learning
  optimal reserve prices for generalized second price auctions. We
  describe two algorithms for this task: one based on density
  estimation, and a novel algorithm benefiting from solid theoretical
  guarantees and with a very favorable running-time complexity of
  $O(n S \log (n S))$, where $n$ is the sample size and $S$ the number
  of slots. Our theoretical guarantees are more favorable than those
  previously presented in the literature. Additionally, we show that
  even if bidders do not play at an equilibrium, our second algorithm
  is still well defined and minimizes a quantity of interest.  To our
  knowledge, this is the first attempt to apply learning algorithms to
  the problem of reserve price optimization in GSP auctions. Finally,
  we present the first convergence analysis of empirical equilibrium
  bidding functions to the unique symmetric Bayesian-Nash equilibrium
  of a GSP.

\end{abstract}

\section{INTRODUCTION}

The Generalized Second-Price (GSP) auction is currently the standard
mechanism used for selling sponsored search advertisement. As
suggested by the name, this mechanism generalizes the standard
second-price auction of \cite{vickrey2012} to multiple items. In the
case of sponsored search advertisement, these items correspond to ad
slots which have been ranked by their position. Given this ranking,
the GSP auction works as follows: first, each advertiser places a bid;
next, the seller, based on the bids placed, assigns a score to each
bidder. The highest scored advertiser is assigned to the slot in the
best position, that is, the one with the highest likelihood of being
clicked on. The second highest score obtains the second best item and
so on, until all slots have been allocated or all advertisers have
been assigned to a slot. As with second-price auctions, the bidder's
payment is independent of his bid. Instead, it depends solely on the
bid of the advertiser assigned to the position below.

In spite of its similarity with second-price auctions, the GSP auction
is not an incentive-compatible mechanism, that is, bidders have an
incentive to lie about their valuations. This is in stark contrast
with second-price auctions where truth revealing is in fact a dominant
strategy. It is for this reason that predicting the behavior of
bidders in a GSP auction is challenging. This is further worsened by
the fact that these auctions are repeated multiple times a day. The
study of all possible equilibria of this repeated game is at the very
least difficult. While incentive compatible generalizations of the
second-price auction exist, namely the Vickrey-Clark-Gloves (VCG)
mechanism, the simplicity of the payment rule for GSP auctions as well
as the large revenue generated by them has made the adoption of VCG
mechanisms unlikely.

Since its introduction by Google, GSP auctions have generated billions
of dollars across different online advertisement companies. It is
therefore not surprising that it has become a topic of great interest
for diverse fields such as Economics, Algorithmic Game Theory and more
recently Machine Learning.

The first analysis of GSP auctions was carried out independently by
\cite{edelmanostrovsky05} and \cite{Varian07}.  Both publications
considered a full information scenario, that is one where the
advertisers' valuations are publicly known. This assumption is weakly
supported by the fact that repeated interactions allow advertisers to
infer their adversaries' valuations. \cite{Varian07} studied the
so-called Symmetric Nash Equilibria (SNE) which is a subset of the
Nash equilibria with several favorable properties. In particular,
Varian showed that any SNE induces an efficient allocation, that is an
allocation where the highest positions are assigned to advertisers
with high values. Furthermore, the revenue earned by the seller when
advertisers play an SNE is always at least as much as the one obtained by
VCG. The authors also presented some empirical results showing that some
bidders indeed play by using an SNE. However, no theoretical
justification can be given for the choice of this subset of equilibria
\citep{Borgers07, Edelman10}. A finer analysis of the full information
scenario was given by \cite{Lucier12}. The authors proved that,
excluding the payment of the highest bidder, the revenue achieved at
any Nash equilibrium is at least one half that of the VCG auction.

Since the assumption of full information can be unrealistic, a more
modern line of research has instead considered a Bayesian scenario for this
auction. In a Bayesian setting, it is assumed that advertisers'
valuations are i.i.d.\ samples drawn from a common distribution. 
\cite{Gomes} characterized all symmetric Bayes-Nash equilibria and
showed that any symmetric equilibrium must be efficient. This work was
later extended by \cite{SunZhou} to account for the quality score of
each advertiser. The main contribution of this work was the design of
an algorithm for the crucial problem of revenue optimization for the
GSP auction. \cite{lahaiepennock2007} studied
different \emph{squashing} ranking rules for advertisers commonly used
in practice and showed that none of these rules are necessarily optimal
in equilibrium. This work is complemented by the simulation analysis
of \cite{Vorobeychik09} who quantified the distance from equilibrium
of bidding truthfully. \cite{Lucier12} showed that the GSP
auction with an optimal reserve price achieves at least 1/6 of the
optimal revenue (of any auction) in a Bayesian equilibrium.  More
recently, \cite{ThompsonLeytonBrown} compared
different allocation rules and showed that an \emph{anchoring}
allocation rule is optimal when valuations are sampled i.i.d.\ from a
uniform distribution. With the exception of \cite{SunZhou}, none of
these authors have proposed an algorithm for revenue optimization
using historical data. 

\cite{ZhuWang} introduced a ranking algorithm to learn an optimal
allocation rule. The proposed ranking is a convex combination of a
quality score based on the features of the advertisement as well as a
revenue score which depends on the value of the bids. This work was later
extended in \citep{HeChen} where, in addition to the ranking
function, a behavioral model of the advertisers is learned by the
authors. 

The rest of this paper is organized as follows. In
Section~\ref{sec:model}, we give a learning formulation of the problem
of selecting reserve prices in a GSP auction. In
Section~\ref{sec:previous}, we discuss previous work related to this
problem. Next, we present and analyze two learning algorithms for this
problem in Section~\ref{sec:algorithms}, one based on density
estimation extending to this setting an algorithm of
\cite{guerre2000}, and a novel discriminative algorithm taking into
account the loss function and benefiting from favorable learning
guarantees. Section~\ref{sec:convergence} provides a convergence
analysis of the empirical equilibrium bidding function to the true
equilibrium bidding function in a GSP. On its own, this result is of
great interest as it justifies the common assumption of buyers playing
a symmetric Bayes-Nash equilibrium. Finally, in
Section~\ref{sec:experiments}, we report the results of experiments
comparing our algorithms and demonstrating in particular the benefits
of the second algorithm.

\section{MODEL}
\label{sec:model}

For the most part, we will use the model defined by \cite{SunZhou} for
GSP auctions with incomplete information. We consider $N$ bidders
competing for $S$ slots with $N \geq S$. Let $v_i \in [0,1] $ and $b_i
\in [0,1]$ denote the per-click valuation of bidder $i$ and his bid
respectively.  Let the position factor $c_s \in [0,1]$ represent the
probability of a user noticing an ad in position $s$ and let $e_i \in
[0,1]$ denote the expected click-through rate of advertiser $i$. That
is $e_i$ is the probability of ad $i$ being clicked on given that it
was noticed by the user. We will adopt the common assumption that $c_s
> c_{s + 1}$ \citep{Gomes, lahaiepennock2007,
  SunZhou,ThompsonLeytonBrown}. Define the score of bidder $i$ to be
$s_i = e_i v_i$.  Following \cite{SunZhou}, we assume that $s_i$ is an
i.i.d.\ realization of a random variable with distribution $F$ and
density function $f$. Finally, we assume that advertisers bid in an
efficient symmetric Bayes-Nash equilibrium. This is motivated by the
fact that even though advertisers may not infer what the valuation of
their adversaries is from repeated interactions, they can certainly
estimate the distribution $F$.

Define $\pi\colon s \mapsto \pi(s)$ as the function mapping slots to
advertisers, i.e.\ $\pi(s) = i$ if advertiser $i$ is allocated to
position $s$. For a vector $\mat x = (x_1, \ldots, x_N) \in \Rset^N$,
we use the notation $x^{(s)} := x_{\pi(s)}$. Finally, denote by $r_i$
the reserve price for advertiser $i$. An advertiser may participate in
the auction only if $b_i \geq r_i$.  In this paper we present an
analysis of the two most common ranking rules \citep{Taolitrev}:
\begin{enumerate}
\setlength{\itemsep}{0pt}
\item Rank-by-bid. Advertisers who bid above their reserve price are
ranked in descending order of their bids and the payment of advertiser
$\pi(s)$ is equal to $\max(r^{(s)} , b^{(s + 1)})$.

\item Rank-by-revenue. Each advertiser is assigned a quality score
$q_i := q_i(b_i) = e_i b_i \Ind_{b_i \geq r_i}$ and the ranking is done by sorting
these scores in descending order. The payment of advertiser $\pi(s)$
is given by $\max \big(r^{(s)}, \frac{q^{(s + 1)}}{e^{(s)}}\big)$.

\end{enumerate}
In both setups, only advertisers bidding above their reserve price are
considered. Notice that rank-by-bid is a particular case of
rank-by-revenue where all quality scores are equal to $1$. Given a
vector of reserve prices $\mat r$ and a bid vector $\mat b$, we define
the revenue function to be
\begin{multline*}
 \text{Rev}(\mat r, \mat b) 
\\= \sum_{s=1}^S c_s \Big(\frac{q^{(s + 1)}}{e^{(s)}}
 \Ind_{q^{(s + 1)} \geq e^{(s)}r^{(s)}} 
+ r^{(s)} \Ind_{q^{(s + 1)} < e^{(s)}r^{(s)} \leq q^{(s)}} \Big)
\end{multline*}
Using the notation of \cite{MohriMunoz}, we define the loss function 
\begin{equation*}
  L(\mat r, \mat b) = -\text{Rev}(\mat r, \mat b).
\end{equation*}
Given an i.i.d.\ sample $\mathcal{S} = (\mat b_1, \ldots, \mat b_n)$
of realizations of an auction, our objective will be to find a reserve
price vector $\mat r^*$ that maximizes the expected
revenue. Equivalently, $\mat r^*$ should be a solution of the
following optimization problem:
\begin{equation}
\label{eq:opttrue}
  \min_{\mat r \in [0,1]^N} \E_{\mat b} [L(\mat r, \mat b)].
\end{equation}

\section{PREVIOUS WORK}
\label{sec:previous}

It has been shown, both theoretically and empirically, that reserve
prices can increase the revenue of an auction
\citep{Myerson,Ostrovskyfield}. The choice of an appropriate reserve
price therefore becomes crucial. If it is chosen too low, the seller
might lose some revenue. On the other hand, if it is set too high,
then the advertisers may not wish to bid above that value and the
seller will not obtain any revenue from the auction.

\cite{MohriMunoz}, \cite{pardoe}, and \cite{gentile} have given
learning algorithms that estimate the optimal reserve price for a
second-price auction in different information scenarios. The scenario
we consider is most closely related to that of \cite{MohriMunoz}. An
extension of this work to the GSP auction, however, is not
straightforward. Indeed, as we will show later, the optimal reserve
price vector depends on the distribution of the advertisers'
valuation. In a second-price auction, these valuations are observed
since the corresponding mechanism is an incentive-compatible. This
does not hold for GSP auctions. Moreover, for second-price auctions,
only one reserve price had to be estimated. In contrast, our model
requires the estimation of up to $N$ parameters with intricate
dependencies between them.

The problem of estimating valuations from observed bids in a
non-incentive compatible mechanism has been previously
analyzed. \cite{guerre2000} described a way of estimating valuations
from observed bids in a first-price auction. We will show that this
method can be extended to the GSP auction. The rate of convergence of
this algorithm, however, in general will be worse than the standard
learning rate of $O\big(\frac{1}{\sqrt n} \big)$.

\cite{SunZhou} showed that, for advertisers playing an efficient
equilibrium, the optimal reserve price is given by $r_i =
\frac{\overline r}{e_i}$ where $\overline r$ satisfies
\begin{equation*}
 \overline r = \frac{1 - F(\overline r)}{f(\overline r)}.
\end{equation*}
The authors suggest learning $\overline r$ via a maximum likelihood
technique over some parametric family to estimate $f$ and $F$, and to
use these estimates in the above expression. There are two main
drawbacks for this algorithm. The first is a standard problem of
parametric statistics: there are no guarantees on the convergence of
their estimation procedure when the density function $f$ is not part
of the parametric family considered. While this problem can be
addressed by the use of a non-parametric estimation algorithm such as
kernel density estimation, the fact remains that the function $f$ is
the density for the unobservable scores $s_i$ and therefore cannot be
properly estimated. The solution proposed by the authors assumes that
the bids in fact form a perfect SNE and so advertisers'
valuations can be recovered using the process described by
\cite{Varian07}. There is however no justification for this assumption
and, in fact, we show in Section~\ref{sec:experiments} that bids
played in a Bayes-Nash equilibrium do not in general form a SNE.

\section{LEARNING ALGORITHMS}
\label{sec:algorithms}

Here, we present and analyze two algorithms for learning the optimal
reserve price for a GSP auction when advertisers play a symmetric
equilibrium. 

\subsection{DENSITY ESTIMATION ALGORITHM}
\label{sec:density}

First, we derive an extension of the algorithm of
\cite{guerre2000} to GSP auctions. To do so, we first derive a formula
for the bidding strategy at equilibrium. Let $z_s(v)$ denote the
probability of winning position $s$ given that the advertiser's
valuation is $v$. It is not hard to verify that
\begin{equation*}
  z_s(v) = \binom{N-1}{s-1}(1 - F(v))^{s-1}F^{p}(v),
\end{equation*}
where $p = N-s$. Indeed, in an efficient equilibrium, the bidder with
the $s$-th highest valuation must be assigned to the $s$-th highest
position. Therefore an advertiser with valuation $v$ is assigned to
position $s$ if and only if $s-1$ bidders have a higher
valuation and $p$ have a lower valuation. 

For a rank-by-bid auction, \cite{Gomes} showed the following results.

\begin{theorem}[\cite{Gomes}]
  A GSP auction has a unique efficient symmetric Bayes-Nash
  equilibrium with bidding strategy $\beta$ if and only if
  $\beta$ is strictly increasing and satisfies the following
  integral equation:
\begin{align}
\label{eq:volterra}
& \sum_{s=1}^S c_s \int_0^v \frac{dz_s(t)}{dt} t dt \\ 
& = \sum_{s=1}^S c_s \coeff (1 - F(v))^{s-1} \mspace{-5mu} \int_0^v \mspace{-5mu}\beta(t) p
F^{p-1}(t)f(t) dt.\nonumber
\end{align}
Furthermore, the optimal reserve price $r^*$ satisfies
\begin{equation}
\label{eq:optreserve}
  r^* = \frac{1 - F(r^*)}{f(r^*)}.
\end{equation}
\end{theorem}

The authors show that, if the click probabilities $c_s$ are
sufficiently diverse, then, $\beta$ is guaranteed to be strictly
increasing. When ranking is done by revenue, \cite{SunZhou}
gave the following theorem.

\begin{theorem}[\cite{SunZhou}]
 Let $\beta$ be defined by the previous theorem. If advertisers bid in
 a Bayes-Nash equilibrium then $b_i =
 \frac{\beta(v_i)}{e_i}$. Moreover, the optimal reserve price vector
 $\mat r^*$ is given by $r^*_i = \frac{\overline r}{e_i}$ where
 $\overline r$ satisfies equation \eqref{eq:optreserve}.
\end{theorem}

We are now able to present the foundation of our first algorithm.
Instead of assuming that the bids constitute an SNE as in
\citep{SunZhou}, we follow the ideas of \cite{guerre2000} and infer
the scores $s_i$ only from observables $b_i$. Our result is presented
for the rank-by-bid GSP auction but an extension to the
rank-by-revenue mechanism is trivial.

\begin{lemma}
  Let $v_1, \ldots, v_n$ be an i.i.d.\ sample of valuations from
  distribution $F$ and let $b_i = \beta(v_i)$ be the bid played at
  equilibrium. Then the random variables $b_i$ are i.i.d.\ with
  distribution $G(b) = F(\beta^{-1}(b))$ and density
  $g(b) = \frac{ f(\beta^{-1}(b)) }{\beta'
    (\beta^{-1}(b))}$. Furthermore,
\begin{flalign}
\label{eq:inversevolterra}
& v_i = \beta^{-1}(b_i)&\\
& = \frac{\sum_{s=1}^S c_s \coeff (1 - G(b_i))^{s-1} b_i p
  G(b_i)^{p-1}
  g(b_i)}{\sum_{s=1}^S c_s \coeff \frac{d \overline z}{d b}(b_i)}\nonumber \\
& \!-\! \frac{\sum_{s = 1}^S \! c_s (s \!-\! 1) (1 \!-\! G(b_i))^{s-2}
  g(b_i) \!\!  \int_0^{b_i} \!\! p G(u)^{p-1}u g(u) du} {\sum_{s=1}^S
  c_s \coeff \frac{d \overline z}{d b}(b_i)}, \nonumber
\end{flalign}
where $\overline z_s(b) := z_s(\beta^{-1}(b))$ and is given by
$\coeff (1 - G(b))^{s-1} G(b)^{p-1}$. 
\end{lemma}
\begin{proof}
  By definition, $b_i = \beta(v_i)$ is a function of only $v_i$. Since
  $\beta$ does not depend on the other samples either, it follows that
  $(b_i)_{i=1}^N$ must be an i.i.d.\ sample. Using the fact that
  $\beta$ is a strictly increasing function we also have $G(b) = P(b_i
  \leq b) = P(v_i \leq \beta^{-1}(b)) = F(\beta^{-1}(b))$ and a simple
  application of the chain rule gives us the expression for the
  density $g(b)$.  To prove the second statement observe that by the
  change of variable $v = \beta^{-1}(b)$, the right-hand side of
  \eqref{eq:volterra} is equal to
\begin{align*}
& \sum_{s=1}^S \coeff (1 - G(b))^{s-1} \!\!\int_0^{\beta^{-1}(b)} \!\!\!p \beta(t)
F^{p-1} (t) f(t) dt  \\
= & \sum_{s=1}^S \coeff(1 - G(b))^{s-1} \int_0^b p u G(u)^{p-1}(u) g(u) du.
\end{align*}
The last equality follows by the change of variable $t = \beta(u)$
and from the fact that $g(b) =
\frac{f(\beta^{-1}(b))}{\beta'(\beta^{-1}(b))}$. The same change of
variables  applied to the left-hand side of \eqref{eq:volterra}
yields the following integral equation:
\begin{align*}
\lefteqn{ \sum_{s=1}^S \coeff \int_0^b \beta^{-1}(u) \frac{d \overline
    z}{d u } (u) du} \\
& =  \sum_{s=1}^S \coeff(1 - G(b))^{s-1} \!\!\int_0^b  u p G(u)^{p-1}(u) g(u) du.
\end{align*}
Taking the derivative with respect to $b$ of both sides of this
equation and rearranging terms lead to the desired expression.
\end{proof}

The previous Lemma shows that we can recover the valuation of an
advertiser from its bid. We therefore propose the following algorithm
for estimating the value of $\overline r$.
\begin{enumerate}
\item Use the sample $\mathcal{S}$ to estimate $G$ and $g$. 
\item Plug this estimates in \eqref{eq:inversevolterra} to obtain
  approximate samples from the distribution $F$. 
\item Use the approximate samples to find estimates $\widehat f$ and
  $\widehat F$ of the valuations density and cumulative distribution
  functions respectively.
\item Use $\widehat F$ and $\widehat f$ to estimate $\overline r$.
\end{enumerate}
In order to avoid the use of parametric methods, a kernel density
estimation algorithm can be used to estimate $g$ and $f$. While this
algorithm addresses both drawbacks of the algorithm proposed by
\cite{SunZhou}, it can be shown \citep{guerre2000}[Theorem 2] that if
$f$ is $R$ times continuously differentiable, then, after seeing $n$
samples, $\|f - \widehat f\|_\infty$ is in $\Omega \big(
\frac{1}{n^{R/(2R +3)}} \big)$ independently of the algorithm used to
estimate $f$. In particular, note that for $R = 1$ the rate is in
$\Omega \big( \frac{1}{n^{1/4}} \big)$. This unfavorable rate of
convergence can be attributed to the fact that a two-step estimation
algorithm is being used (estimation of $g$ and $f$). But, even with
access to bidder valuations, the rate can only be improved to
$\Omega\big(\frac{1}{n^{R/(2 R + 1)}} \big)$
\citep{guerre2000}. Furthermore, a small error in the estimation of $f$
affects the denominator of the equation defining $\overline r$ and can
result in a large error on the estimate of $\overline r$.

\subsection{DISCRIMINATIVE ALGORITHM}
\label{sec:discriminative}

In view of the problems associated with density estimation, we propose
to use empirical risk minimization to find an approximation to the
optimal reserve price.  In particular, we are interested in solving
the following optimization problem:
\begin{equation}
\label{eq:optempfull}
 \min_{\mat r \in [0,1]^N} \sum_{i=1}^n L(\mat r, \mat b_i).
\end{equation}
We first show that, when bidders play in equilibrium, the optimization
problem \eqref{eq:opttrue} can be considerably simplified.
\begin{proposition}
If advertisers play a symmetric Bayes-Nash equilibrium then 
\begin{equation*}
 \min_{\mat r \in [0,1]^N} \E_{\mat b} [L(\mat r, \mat b)] = \min_{r \in
   [0,1]} \E_{\mat b} [\wt L(r, \mat b)],
\end{equation*}
where $\q_i := \q_i(b_i) =  e_i b_i$ and 
\begin{equation*}
  \wt  L(r, \mat b) = - \sum_{s=1}^S  \frac{c_s}{e^{(s)}} \Big(
\q^{(s + 1)} \Ind_{\q^{(s + 1)} \geq r}  +r \Ind_{\q^{(s + 1)} < r \leq \q^{(s)}} \Big).
\end{equation*}
\end{proposition}
\begin{proof}
Since advertisers play a symmetric Bayes-Nash equilibrium, the
optimal reserve price vector $\mat r^*$ is of the form $r^*_i =
\frac{\overline r}{e_i}$. Therefore, letting $D = \{\mat r | r_i =
\frac{r}{e_i} , \ r \in [0,1] \}$ we have $\min_{\mat r \in [0,1]^N} \E_{\mat
  b} [L(\mat r, \mat b)] = \min_{\mat r \in D} \E_{\mat b} [L(\mat r, \mat b)]
$. Furthermore, when restricted to $D$, the objective function $L$ is
given by 
\begin{equation*}
- \sum_{s=1}^S  \frac{c_s}{e^{(s)}} \Big(
q^{(s + 1)} \Ind_{q^{(s + 1)} \geq  r} +  r \Ind_{q^{(s + 1)} < r \leq q^{(s)}} \Big).
\end{equation*}
Thus, we are left with showing that replacing $q^{(s)}$ with $\q^{(s)}$
in this expression does not affect its value. Let $r \geq 0$,
since $q_i = \q_i \Ind_{\q_i \geq r}$, in general the equality
$q^{(s)} = \q^{(s)}$ does not hold. Nevertheless, if $s_0$ denotes the largest index
less than or equal to $S$ satisfying $q^{(s_0)} > 0$, then $\q^{(s)}
\geq r$ for all $s \leq s_0$ and $q^{(s)} = \q^{(s)}$. On the other
hand, for $S \geq s > s_0$, $\Ind_{q^{(s)} \geq r} = \Ind_{\q^{(s)} \geq
  r} = 0$. Thus,
\begin{flalign*}
&\sum_{s=1}^S  \frac{c_s}{e^{(s)}} \Big( q^{(s + 1)} \Ind_{q^{(s + 1)} \geq r}  +
r \Ind_{q^{(s + 1)} < r \leq q^{(s)}} \Big) &\\ 
& = \sum_{s=1}^{s_0}  \frac{c_s}{e^{(s)}} \Big(
q^{(s + 1)} \Ind_{q^{(s + 1)} \geq r}  +
r \Ind_{q^{(s + 1)} < r \leq q^{(s)}} \Big) \\
&= \sum_{s=1}^{s_0}  \frac{c_s}{e^{(s)}} \Big(
\q^{(s + 1)} \Ind_{\q^{(s + 1)} \geq r}  +
r \Ind_{\q^{(s + 1)} < r \leq \q^{(s)}} \Big) \\
&= - \wt  L(r, \mat b),
\end{flalign*}
which completes the proof.
\end{proof}

\begin{figure}[t]
\hspace{1.66cm} \includegraphics[scale=.45]{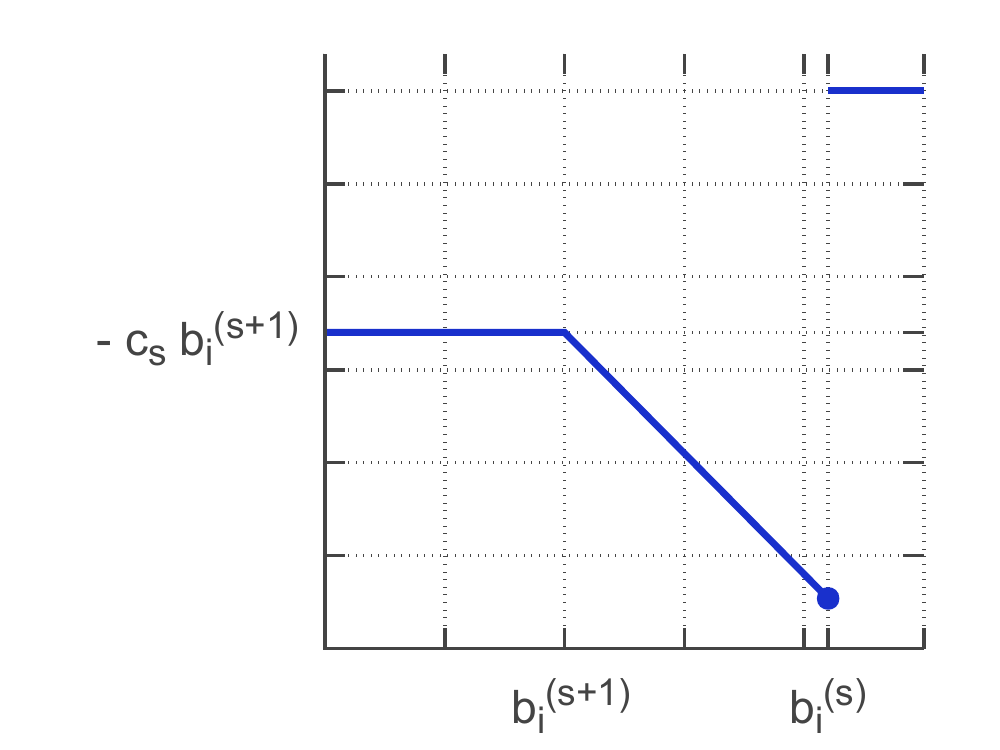}
\caption{Plot of the loss function $L_{i, s}$. Notice that the loss in
  fact resembles a broken ``V'' \label{fig:loss}.}
\end{figure}
In view of this proposition, we can replace the challenging problem of
solving an optimization problem in $\Rset^N$ with solving the following
simpler empirical risk minimization problem
\begin{equation}
\label{eq:optemp}
\min_{r \in [0,1]} \sum_{i=1}^n \wt L(r, \mat b_i)  = \min_{r \in
[0,1]} \sum_{i=1}^n \sum_{s=1}^S L_{s,i} (r, \q^{(s)}, \q^{(s + 1)} ),
\end{equation}
where
$L_{s,i}(r, \q^{(s)}), \q^{(s + 1)}) := - \frac{c_s}{e^{(s)}} (\wt
q_i^{(s + 1)} \Ind_{\q_i^{(s + 1)} \geq r} - r \Ind_{\q_i^{(s +1 )} <
  r \leq \q_i^{(s)}})$.
In order to efficiently minimize this highly non-convex function, we
draw upon the work of \cite{MohriMunoz} on minimization of sums of
$v$-functions.

\begin{definition} A function $V\colon \Rset^3\to \Rset$ is a
\emph{$v$-function} if it admits the following form:
\begin{multline*}
V(r, q_1, q_2)\\
\mspace{-4mu} = \mspace{-4mu} -a^{(1)} \Ind_{r \leq q_2} -a^{(2)}r \Ind_{q_2 < r \leq q_1} +
\Big[ \frac{ r}{\eta} - a^{(3)} \Big]\Ind_{q_1 < r < (1 + \eta) q_1},
\end{multline*}
with $0 \leq a^{(1)}, a^{(2)}, a^{(3)}, \eta \leq \infty$ constants
satisfying $a^{(1)} = a^{(2)} q_2$, $- a^{(2)} q_1\Ind_{\eta > 0} =
\big(\frac{1}{\eta} q_1 - a^{(3)} \big)\Ind_{\eta > 0}$. Under the
convention that $0 \cdot \infty = 0$.
\end{definition}

\begin{algorithm}[t]
\begin{algorithmic}[1]
\REQUIRE Scores $(\widetilde q_i^{(s)})$, $1 \leq n$, $1 \leq s \leq S$.
\STATE Define $(p_{is}^{(1)}, p_{is}^{(2)} ) = (\wt q_i^{(s)},
\wt q_i^{(s+1)})$;  $m = nS$;
\STATE{\strut $\mathcal{N} := \bigcup_{i=1}^n \bigcup_{s=1}^S \{p_{is}^{(1)}, p_{is}^{(2)}\}$;}
\STATE{\strut $(n_1,...,n_{2 m})= {\bf Sort}(\mathcal{N})$;}
\STATE{\strut Set $\mathbf{d}_i:=(d_1,d_2) = \mat 0 $}
\STATE{\strut Set $d_1 = -\sum_{i=1}^n \sum_{s=1}^S\frac{c_s}{e_i} p_{is}^{(2)}$; }
\STATE{\strut Set $r^* = -1$ and $L^* = \infty$}
\FOR{$j = 2, \ldots, 2 m $}
\IF{$n_{j-1} = p_{is}^{(2)}$}
\STATE{$d_1 = d_1 + \frac{c_s}{e_i} p_{is}^{(2)}$; $\; d_2 = d_2 - \frac{c_s}{e_i};$}
\ELSIF{$n_{j-1} = p_{is}^{(1)}$}
\STATE{$d_2 = d_2 + \frac{c_s}{e_s}$}
\ENDIF
\STATE $L = d_1 - n_j d_2 $;
\IF{$L < L^*$}
\STATE $L^* = L$; $r^* = n_j$;
\ENDIF
\ENDFOR
\RETURN $r^*$;
\end{algorithmic}
\caption{Minimization algorithm \label{alg:algorithm}}
\end{algorithm}
As suggested by their name, these functions admit a characteristic ``V
shape''. It is clear from Figure~\ref{fig:loss} that $L_{s,i}$ is a
$v$-function with $a^{(1)} = \frac{c_s}{e^{(s)}} \wt q_i^{(s + 1)}$,
$a^{(2)} = \frac{c_s}{e^{(s)}}$ and $\eta = 0$.  Thus, we can apply
the optimization algorithm given by \cite{MohriMunoz} to minimize
\eqref{eq:optemp} in $O(n S \log n S)$
time. Algorithm~\ref{alg:algorithm} gives the pseudocode of that the
adaptation of this general algorithm to our problem. A proof of the
correctness of this algorithm can be found in \citep{MohriMunoz}.

We conclude this section by presenting learning guarantees for our
algorithm. Our bounds are given in terms of the Rademacher complexity
and the VC-dimension.
\begin{definition}
  Let $\mathcal X$ be a set and let $G := \{g: \mathcal X \to \Rset\}$
  be a family of functions. Given a sample $\mathcal{S} = (x_1,
  \ldots, x_n) \in \mathcal{X}$, the empirical Rademacher complexity
  of $G$ is defined by
\begin{equation*}
\h \Rad_S(G) = \frac{1}{n} \E_\sigma \Big[ \sup_{g \in G} \frac{1}{n}
\sum_{i=1}^n \sigma_i g(x_i) \Big],
\end{equation*}
where $\sigma_i$s are independent random variables distributed
uniformly over the set $\{-1, 1\}$.
\end{definition}
\begin{proposition}
\label{prop:guarantees}
Let $\mathfrak{m} = \min_{i} e_i > 0 $ and
$\mathfrak{M} = \sum_{s=1}^S c_s$. Then, for any $\delta > 0$, with
probability at least $1 - \delta$ over the draw of a sample
$\mathcal{S}$ of size $n$, each of the following inequalities holds
for all $r \in [0,1]$:
\begin{align}
\label{eq:boundtrue}
\E[\wt L(r, \mat b)] 
& \leq \frac{1}{n} \sum_{i=1}^n \wt L(r, \mat b_i)
+ C(\mathfrak{M},\mathfrak{m}, n, \delta) \\ 
\frac{1}{n} \sum_{i=1}^n  \wt L(r, \mat b_i)
& \leq  \E[\wt L(r, \mat b)] 
+ C(\mathfrak{M}, \mathfrak{m}, n, \delta), \label{eq:boundemp}
\end{align}
where $C(\mathfrak{M}, \mathfrak{m}, n, \delta)
= \frac{1}{\sqrt{n}}
+ \sqrt{\frac{\log(e n)}{n}}
+ \sqrt{\frac{ \mathfrak M \log(1/\delta)}{2 \mathfrak m n}}$.
\end{proposition}
\begin{proof}
  Let $\Psi\colon S \mapsto \sup_{r \in [0,1]} \frac{1}{n}
\sum_{i=1}^n \wt L(r, \mat b_i) - \E[\wt L(r, \mat b)]$. Let
$\mathcal{S}^i$ be a sample obtained from $\mathcal{S}$ by replacing
$\mat b_i$ with $\mat b'_i$. It is not hard to verify that $|
\Psi(\mathcal{S}) - \Psi(\mathcal{S}^i) |\leq \frac{\mathfrak M}{n
\mathfrak m}$. Thus, it follows from a standard learning bound that,
with probability at least $1 - \delta$,
\begin{equation*}
 \E[\wt L(r, \mat b)] \leq \frac{1}{n} \sum_{i=1}^n \wt L(r, \mat b_i) +
\h \Rad_S(\mathcal R) + \sqrt{\frac{\mathfrak M \log(1/\delta)}{ 2 \mathfrak{m} n}},
\end{equation*}
where
$\mathcal{R} = \{\overline L_r : \mat b \mapsto \wt L(r, \mat b) | r
\in [0,1]\}$.
We proceed to bound the empirical Rademacher complexity of the class
$\mathcal{R}$. For $q_1 > q_2 \geq 0$ let
$\overline L(r,q_1, q_2) = q_2 \Ind_{q_2 > r} + r \Ind_{q_1 \geq r
  \geq q_2}$. By definition of the Rademacher complexity we can write
\begin{align*}
\h \Rad_S(\mathcal{R})
& = \frac{1}{n} \E_\sigma \Big[ \sup_{r \in [0,1]} \sum_{i=1}^n \sigma_i
\overline L_r(\mat b_i)\Big]
\\
& = \frac{1}{n} \E_\sigma\Big[ \sup_{r \in [0,1]} \sum_{i=1}^n \sigma_i
\sum_{s=1}^S \frac{c_s}{e_s} \overline L(r, \q_i^{(s)}, \q_i^{(s + 1)})\Big] \\
& \leq \frac{1}{n} \E_\sigma \Big[ \sum_{s=1}^S \sup_{r \in [0,1]} \sum_{i=1}^n \sigma_i
 \psi_s(\overline L(r, \q_i^{(s)}, \q_i^{(s + 1)}))\Big],
\end{align*}
where $\psi_s$ is the $\frac{c_s}{\mathfrak m}$-Lipschitz function
mapping $x \mapsto \frac{c_s}{e^{(s)}} x$. Therefore, by Talagrand's
contraction lemma~\citep{LedouxTalagrand91}\ignore{\citep{mohribook}},
the last term is bounded by
\begin{equation*}
\sum_{s=1}^S \frac{c_s}{n \mathfrak{m}} \E_\sigma \!\sup_{r \in [0,1]}
\sum_{i=1}^n \sigma_i  \overline L(r, \q_i^{(s)}, \q_i^{(s + 1)}) 
\! = \! \sum_{s=1}^S \frac{c_s}{\mathfrak m} \h \Rad_{\mathcal{S}_s}(\widetilde{\mathcal{R}}),
\end{equation*}
where $\mathcal{S}_s = \big( (\q_1^{(s)}, \q_1^{(s + 1)}), \ldots,
(\q_n^{(s)}, \q_n^{(s + 1)}) \big)$ and $\widetilde{\mathcal{R}} :=
\{\overline L(r, \cdot, \cdot) | r \in [0,1]\}$. The loss $\overline
L(r, \q^{(s)}, \q^{(s + 1)})$ in fact evaluates to the negative revenue
of a second-price auction with highest bid $\q^{(s)}$ and second
highest bid $\q^{(s + 1)}$ \citep{MohriMunoz}. Therefore, by Propositions
9 and 10 of \cite{MohriMunoz} we can write
\begin{align*}
\h \Rad_{\mathcal{S}_s}(\widetilde{\mathcal{R}}) 
&\leq \frac{1}{n} \E_\sigma\Big[\sup_{r \in [0,1]}\sum_{i=1}^n r
\sigma_i \Big] +  \sqrt{\frac{2 \log en}{n}}\\
& \leq \Big(\frac{1}{\sqrt{n}} + \sqrt{\frac{2 \log en }{n}} \Big),
\end{align*}
which concludes the proof.
\end{proof}
\begin{corollary}
  Under the hypotheses of Proposition~\ref{prop:guarantees}, let
  $\widehat r$ denote the empirical minimizer and $r^*$ the minimizer
  of the expected loss. Then, for any $\delta > 0$, with probability at
  least $1 - \delta$, the following inequality holds:
\begin{equation*}
  \E[\wt L(\widehat r, \mat b)] - \E[\wt L(r^*, \mat b)] \\
\leq 2 C \Big(\mathfrak{M}, \mathfrak{m}, n, \frac{\delta}{2}\Big).
\end{equation*}
\begin{proof}
By the union bound, \eqref{eq:boundtrue} and
\eqref{eq:boundemp} hold simultaneously with probability at least $1 -
\delta$ if $\delta$ is replaced by $\delta/2$ in those
expression. Adding both inequalities and using the fact that $\h r$ is
an empirical minimizer yields the result.
\end{proof}
\end{corollary}

It is worth noting that our algorithm is well defined whether or not
the buyers bid in equilibrium. Indeed, the algorithm consists of the
minimization over $r$ of an observable quantity. While we can
guarantee convergence to a solution of \eqref{eq:opttrue} only when
buyers play a symmetric BNE, our algorithm will still find an
approximate solution to
\begin{equation*}
  \min_{r \in [0,1]} \E_{\mathbf{b}}[L(r, \mathbf{b})],
\end{equation*}
which remains a quantity of interest that can be close to
\eqref{eq:opttrue} if buyers are close to the equilibrium.

\section{CONVERGENCE OF EMPIRICAL EQUILIBRIA}
\label{sec:convergence}

A crucial assumption in the study of GSP auctions, including this
work, is that advertisers bid in a Bayes-Nash equilibrium
\citep{Lucier12,SunZhou}. This assumption is partially justified by
the fact that advertisers can infer the underlying distribution $F$
using as observations the outcomes of the past repeated auctions and
can thereby implement an efficient equilibrium.

In this section, we provide a stronger theoretical justification in
support of this assumption: we quantify the difference between the
bidding function calculated using observed empirical distributions and
the true symmetric bidding function in equilibria. For the sake of
notation simplicity, we will consider only the rank-by-bid GSP auction.

Let $\mathcal{S}_v = (v_1, \ldots, v_n)$ be an i.i.d.\ sample of values
drawn from a continuous distribution $F$ with density function
$f$. Assume without loss of generality that $v_1 \leq \ldots \leq v_n$
and let $\mat v$ denote the vector defined by $\mat v_i = v_i$. 
Let $\h F$ denote the empirical distribution function induced by
$\mathcal{S}_v$ and let $\F \in \Rset^n$ and $\G \in \Rset^n$ be
defined by $\F_i = \h F(v_i) = i/n$ and $\G_i = 1 - \F_i$.

We consider a \emph{discrete} GSP auction where the advertiser's
valuations are i.i.d.\ samples drawn from a distribution $\h F$. In the
event where two or more advertisers admit the same valuation, ties are
broken randomly. Denote by $\h \beta$ the bidding function for this
auction in equilibrium (when it exists). We are interested in
characterizing $\h \beta$ and in providing guarantees on the
convergence of $\h \beta$ to $\beta$ as the sample size increases.

We first introduce the notation used throughout this section.
\begin{definition}
Given a vector $\F \in \Rset^n$, the backwards difference operator $\Delta:
\Rset^{n} \to \Rset^n$ is defined as:
\begin{equation*}
  \Delta \F_i = \F_i - \F_{i-1},
\end{equation*}
for $i > 1$ and $\Delta \F_1 = \F_1$.
\end{definition}
We will denote $ \Delta \Delta \F_i$ by $\Delta^2 \F_i$. Given any $k
\in \mathbb{N}$ and a vector $\F$, the vector $\F^k$ is defined as
$\F^k_i = (\F_i)^k$. Let us now define the discrete analog of the
function $z_s$ that quantifies the probability of winning slot $s$.
\begin{proposition}
\label{prop:empz}
In a symmetric efficient equilibrium of the discrete GSP, the
probability $\h z_s(v)$ that an advertiser with valuation $v$ is
assigned to slot $s$ is given by
\begin{multline*}
 \h z_s(v) \\
= \sum_{j=0}^{N-s} \sum_{k=0}^{s-1} \binom{N-1}{j,k,
 N \! - \! 1 \! -\! j\!- \!k}\frac{\F_{i-1}^j \G_i^k}{(N-j-k)n^{N-1-j-k}},
\end{multline*}
if $v = v_i$ and otherwise by
\begin{equation*}
  \h z_s(v) = \binom{N-1}{s-1} \lim_{v' \rightarrow v^-} \h F(v')^{p} (1 - \h
  F(v))^{s-1} =: \h z_s^-(v),
\end{equation*}
where $p = N - s$.
\end{proposition}
In particular, notice that $\h z_s^-(v_i)$ admits the simple
expression 
\begin{equation*}
\h z_s^-(v_i) = \binom{N-1}{s-1} \F_{i-1}^p \G_{i-1}^{s-1},
\end{equation*}
which is the discrete version of the function $z_s$. On the other
hand, even though $\h z_s(v_i)$ does not admit a closed-form, it is
not hard to show that
\begin{equation}
\label{eq:zsapprox}
 \h z_s(v_i) = \binom{N-1}{s-1} \F_{i-1}^p \G_i^{s-1} +
 O\Big(\frac{1}{n}\Big). 
\end{equation}
Which again can be thought of as a discrete version of $z_s$. The
proof of this and all other propositions in this section are deferred
to the Appendix.
Let us now define the lower triangular matrix $\M(s)$ by:
\begin{equation*}
\M_{ij}(s) = - \binom{N-1}{s-1} \frac{n \Delta \F_j^p \Delta  \G_i^s}{s},
\end{equation*}
for $i > j$ and 
\begin{equation*}
\M_{ii}(s)
  = \!\! \sum_{j=0}^{N-s-1} \sum_{k=0}^{s-1} \binom{N-1}{j,k,
  N\!-\!1\!-\!j\!-\!k} \tfrac{\F_{i-1}^j \G_i^k }{(N\!-\!j\!-\!k)n^{N\!-\!1\!-\!j\!-\!k}}.
\end{equation*}

\begin{proposition}
\label{prop:linear}
If the discrete GSP auction admits a symmetric efficient
equilibrium, then its bidding function $\h \beta$ satisfies $\h
\beta(v_i) = \bbeta_i$, where $\bbeta$ is the solution of
the following linear equation.
\begin{equation}
\label{eq:gsp-linear}
 \M \bbeta = \mat u,
\end{equation}
with $\M = \sum_{s=1}^S c_s \M(s)$ and
$\mat u_i = \sum_{s=1}^S \Big( c_s z_s(v_i)v_i - \sum_{j=1}^i \h
z_s^-(v_j) \Delta \mat v_j \Big)$.
\end{proposition}

To gain some insight about the relationship between $\h \beta$ and
$\beta$, we compare equations \eqref{eq:gsp-linear} and
\eqref{eq:volterra}. An integration by parts of the right-hand side of
\eqref{eq:volterra} and the change of variable $G(v) = 1 - F(v)$ show
that $\beta$ satisfies
\begin{multline}
\label{eq:volterra2}
\sum_{s=1}^S c_s v z_s(v)  -  \int_0^v \frac{d z_s(t)}{dt} t dt \\
= \sum_{s=1}^S c_s \binom{N-1}{s-1} G(v)^{s-1} \int_0^v \beta(t) d F^{p}.
\end{multline}
On the other hand, equation \eqref{eq:gsp-linear} implies that for all $i$
\begin{equation}
\label{eq:gsp-linear2}
\mat u_i = \sum_{s=1}^S c_s \bigg[\M_{ii}(s)  \bbeta_i  
- \binom{N-1}{s-1} \frac{n \Delta \G_i^s}{s} \sum_{j=1}^{i-1} \Delta
  \F_j^p \bbeta_j\bigg].
\end{equation}
Moreover, by Lemma~\ref{lemma:derivative} and
Proposition~\ref{prop:Miiapprox} in the Appendix, the
equalities $ -\frac{n \Delta \G_i^s}{s} = \G_i^{s-1} + O
\big(\frac{1}{n}\big)$ and
\begin{equation*}
  \M_{ii}(s) = \frac{1}{2 n} \binom{N-1}{s-1}  p \F^{p-1}_{i-1}
  \G_i^{s-1} + O \Big(\frac{1}{n^2} \Big),
\end{equation*}
hold. Thus, equation \eqref{eq:gsp-linear2} resembles a numerical
scheme for solving \eqref{eq:volterra2} where the integral on the
right-hand side is approximated by the trapezoidal rule. Equation
\eqref{eq:volterra2} is in fact a Volterra equation of the first kind
with kernel
\begin{equation*}
  K(t,v) = \sum_{s=1}^S \binom{N-1}{s-1} G(v)^{s-1} p F^{p-1}(t).
\end{equation*}
Therefore, we could benefit from the extensive literature on the
convergence analysis of numerical schemes for this type of equations
\citep{baker1977, kress1989,linz1985}. However, equations of the first
kind are in general ill-posed problems \citep{kress1989}, that is
small perturbations on the equation can produce large errors on the
solution. When the kernel $K$ satisfies $\min_{t \in [0,1]} K(t,t) >
0$, there exists a standard technique to transform an equation of the
first kind to an equation of the second kind, which is a well posed
problem. Thus, making the convergence analysis for these types of
problems much simpler. The kernel function appearing in
\eqref{eq:volterra2} does not satisfy this property and therefore these results are not applicable to our scenario.  To the best of our
knowledge, there exists no quadrature method for solving 
Volterra equations of the first kind with vanishing kernel.

In addition to dealing with an uncommon integral equation, we need to
address the problem that the elements of \eqref{eq:gsp-linear} are not
exact evaluations of the functions defining \eqref{eq:volterra2} but
rather stochastic approximations of these functions. Finally, the
grid points used for the numerical approximation are also random.

\begin{figure}[t]
\centering
\includegraphics[scale=.45]{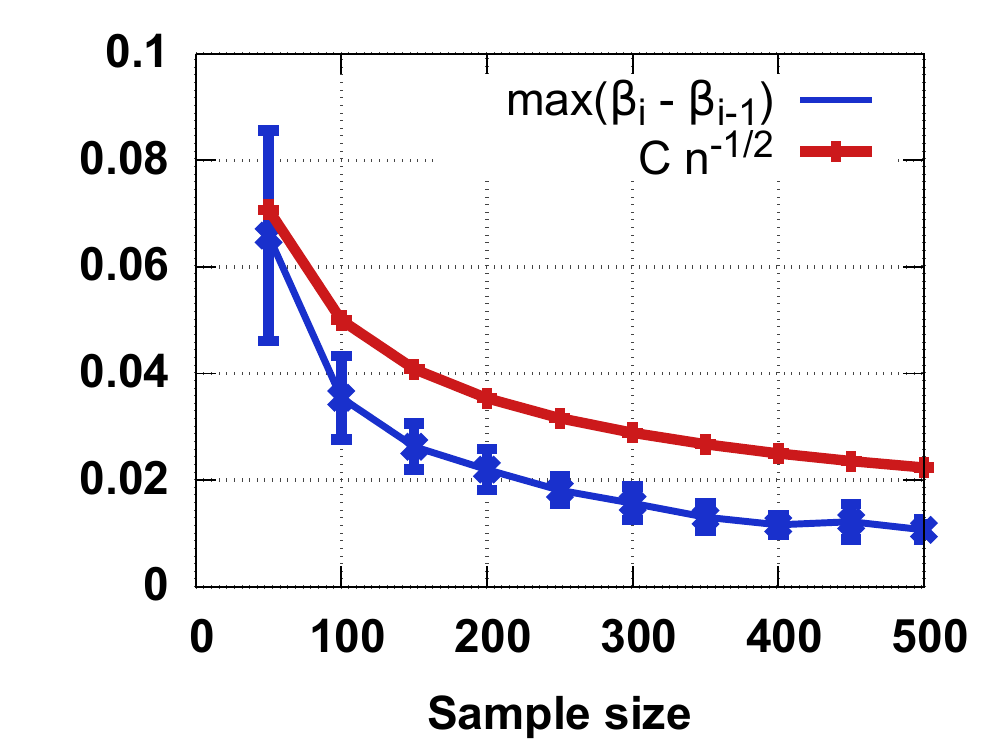} \\
\caption{(a) Empirical verification of
  Assumption~\ref{assum:smooth}. The blue line corresponds to the
  quantity $\max_{i} \Delta \bbeta_i$. In red we plot the desired
  upper bound for $C = 1/2$. \label{fig:diffs}}
\end{figure}

In order to prove convergence of the function $\h \beta$ to $\beta$ we
will make the following assumptions
\begin{assumption}
\label{assum:dens}
There exists a constant $c > 0$ such that $f(x) > c$ for all $x \in
[0,1]$. 
\end{assumption}
This assumption is needed to ensure that the difference between
consecutive samples $v_i - v_{i-1}$ goes to $0$ as $n \rightarrow
\infty$, which is a necessary condition for the convergence of any
numerical scheme. 
\begin{assumption}
\label{assum:smooth}
The solution $\bbeta$ of \eqref{eq:gsp-linear} satisfies $v_i, \bbeta_i
\geq 0$ for all $i$ and $\max_{i \in 1, \ldots, n} \Delta \bbeta_i
\leq \frac{C}{\sqrt{n}}$, for some universal constant $C$.
\end{assumption} 
Since $\bbeta_i$ is a bidding strategy in equilibrium, it is
reasonable to expect that $v_i \geq \bbeta_i \geq 0$. On the other
hand, the assumption on $\Delta \bbeta_i$ is related to the smoothness
of the solution. If the function $\beta$ is smooth, we should expect
the approximation $\h \beta$ to be smooth too. Both assumptions can in
practice be verified empirically, Figure~\ref{fig:diffs} depicts the
quantity $\max_{i \in 1, \ldots, n} \Delta \bbeta_i$ as a function of
the sample size $n$. 
\begin{assumption}
\label{assum:smoothreal}
The solution $\beta$ to \eqref{eq:volterra} is twice continuously
differentiable. 
\end{assumption}
This is satisfied if for instance the distribution function $F$ is
twice continuously differentiable. We can now present our main result.
\begin{theorem}
\label{th:convergence}
If Assumptions~\ref{assum:dens}, \ref{assum:smooth} and
\ref{assum:smoothreal} are satisfied, then, for any $\delta > 0$, with
probability at least $1 - \delta$ over the draw of a sample of size
$n$, the following bound holds for all $i \in [1, n]$:
\begin{equation*}
  | \h \beta(v_i) - \beta(v_i) | \leq
e^C\bigg[\frac{\log(\frac{2}{\delta})^{\frac{N}{2}}}{\sqrt{n}} q\Big(n, \frac{2}{\delta}\Big)^3 +
\frac{C q(n, \frac{2}{\delta})}{n^{3/2}}\bigg].
\end{equation*}
where $q(n, \delta) = \frac{2}{c}\log(nc/2\delta)$ with $c$ defined in
Assumption~\ref{assum:dens}, and where $C$ is a universal constant.
\end{theorem}

The proof of this theorem is highly technical, thus, we defer it to
Appendix~\ref{sec:convergence-proof}.
\begin{figure}[t]
\centering
\includegraphics[scale=.5]{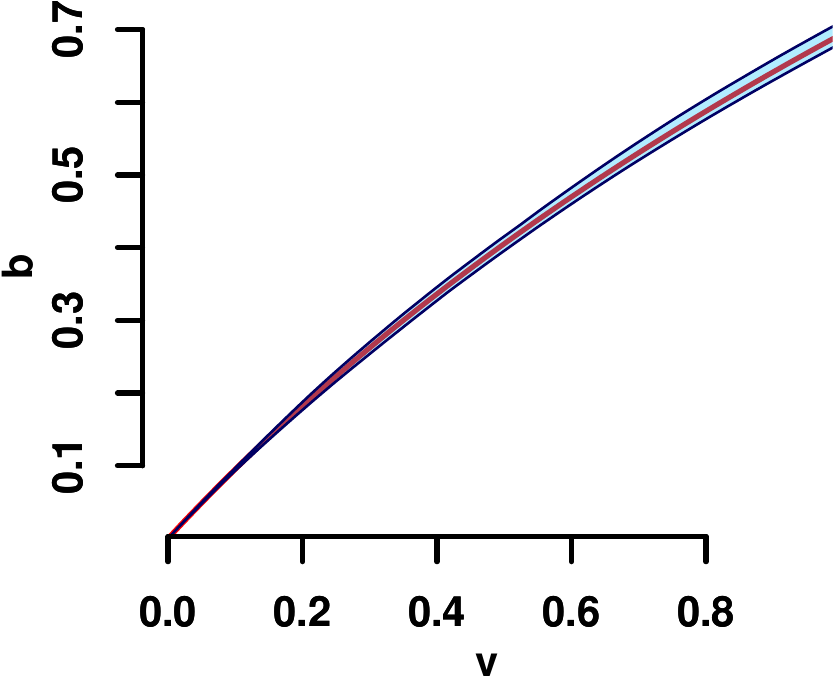}
\caption{Approximation of the empirical
  bidding function $\h \beta$ to the true solution $\beta$. The true
  solution is shown in red and the shaded region represents the
  confidence interval of $\h \beta$ when simulating the discrete $GSP$
  10 times with a sample of size $200$. Where  $N=3$, $S = 2$, $c_1 =
  1$, $c_2 = 0.5$ and bids were sampled uniformly from $[0,1]$}
\end{figure}

\section{EXPERIMENTS}
\label{sec:experiments}

Here we present preliminary experiments showing the advantages of our
algorithm. We also present empirical evidence showing that the
procedure proposed in \cite{SunZhou} to estimate valuations from bids
is incorrect. In contrast, our density estimation algorithm correctly
recovers valuations from bids in equilibrium. 

\subsection{SETUP}

Let $F_1$ and $F_2$ denote the distributions of two truncated
log-normal random variables with parameters $\mu_1 = \log(.5)$,
$\sigma_1 = .8$ and $\mu_2 = \log(2)$, $\sigma=.1$; the mixture
parameter was set to $1/2$ . Here, $F_1$ is truncated to have support
in $[0,1.5]$ and the support of $F_2 = [0, 2.5]$. We consider a GSP
with $N = 4 $ advertisers with $S = 3$ slots and position factors $c_1
= 1$, $c_2 = ,45$ and $c_3 =1$.  Based on the results of
Section~\ref{sec:convergence} we estimate the bidding function $\beta$
with a sample of 2000 points and we show its plot in Figure~\ref{fig:beta}.
\begin{figure}[t]
\centering
\includegraphics[scale=.4]{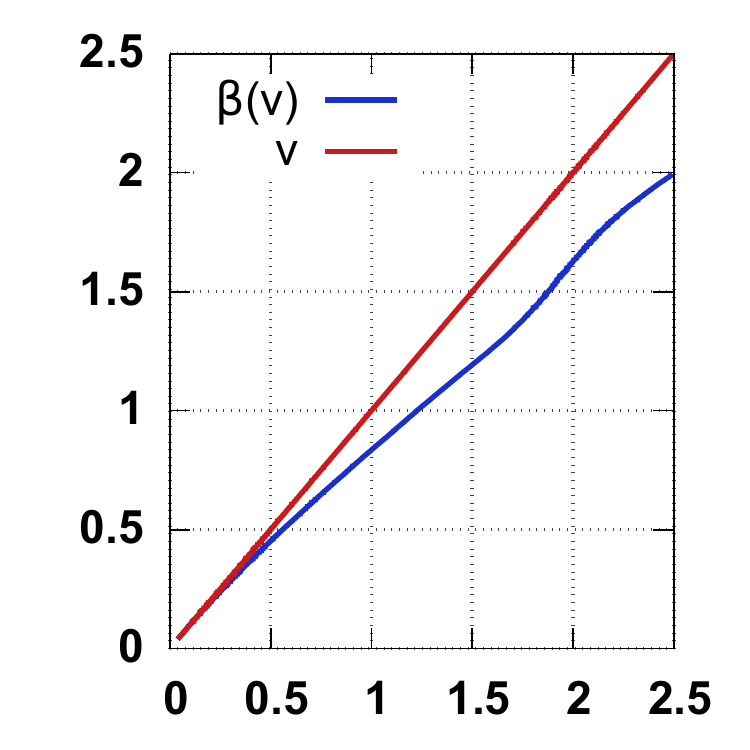}
\caption{Bidding function for our experiments in blue and
identity function in red.}
\label{fig:beta}
\end{figure}
We proceed to evaluate the method proposed by \cite{SunZhou} for
recovering advertisers valuations from bids in equilibrium. The
assumption made by the authors is that the advertisers play a SNE in
which case valuations can be inferred by solving a simple system of
inequalities defining the SNE \citep{Varian07}. Since the authors do not
specify which SNE the advertisers are playing we select the one that
solves the SNE conditions with equality. 

We generated a sample $\mathcal S$ consisting of $n = 300$
i.i.d. outcomes of our simulated auction. Since $N = 4$, the effective
size of this sample is of $1200$ points. We generated the outcome bid
vectors $\mat b_i , \ldots, \mat b_n$ by using the equilibrium bidding
function $\beta$. Assuming that the bids constitute a SNE we estimated
the valuations and Figure~\ref{fig:hist} shows an histogram of the
original sample as well as the histogram of the estimated
valuations. It is clear from this figure that this procedure does not
accurately recover the distribution of the valuations. By contrast, the
histogram of the estimated valuations using our density estimation
algorithm is shown in Figure~\ref{fig:hist}(c). The kernel function
used by our algorithm was a triangular kernel given by $K(u) = (1 -
|u|) \Ind_{|u| \leq 1}$. Following the experimental setup of
\cite{guerre2000} the bandwidth $h$ was set to $h = 1.06 \h \sigma
n^{1/5}$, where $\h \sigma $ denotes the standard deviation of the
sample of bids.

Finally, we use both our density estimation algorithm and
discriminative learning algorithm to infer the optimal value of
$r$. To test our algorithm we generated a test sample of size $n =
500$ with the procedure previously described. The results are shown in
Table~\ref{tab:results}.
\begin{table}[h]
 \centering
\begin{tabular}[h]{c|c}

 Density estimation & Discriminative \\
\hline
  1.42 $\pm$  0.02  & 1.85 $\pm$ 0.02
\end{tabular}
\caption{Mean revenue for our two algorithms.}
\label{tab:results}
\vspace{-10pt}
\end{table}
\begin{figure}[t]
  \centering
\begin{tabular}{c}
(a) \includegraphics[scale=.34]{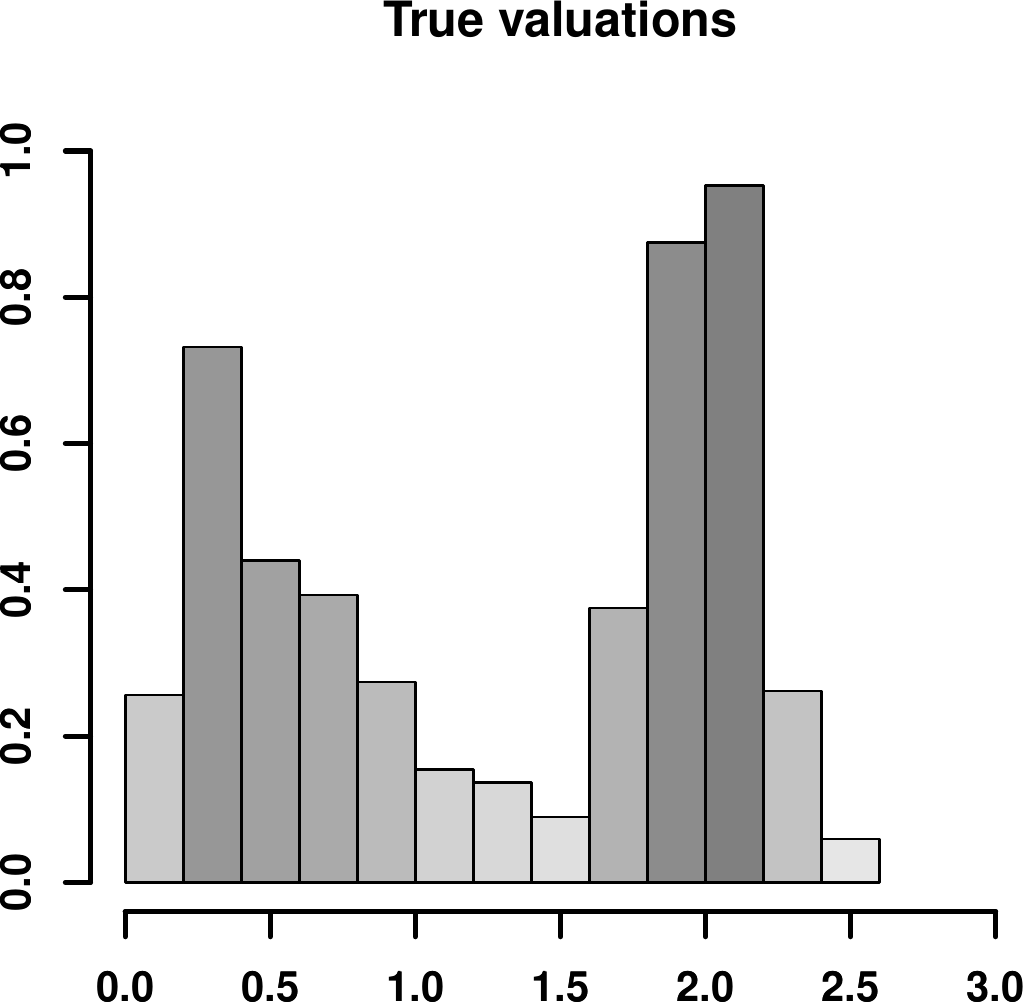} \\
(b)\includegraphics[scale=.34]{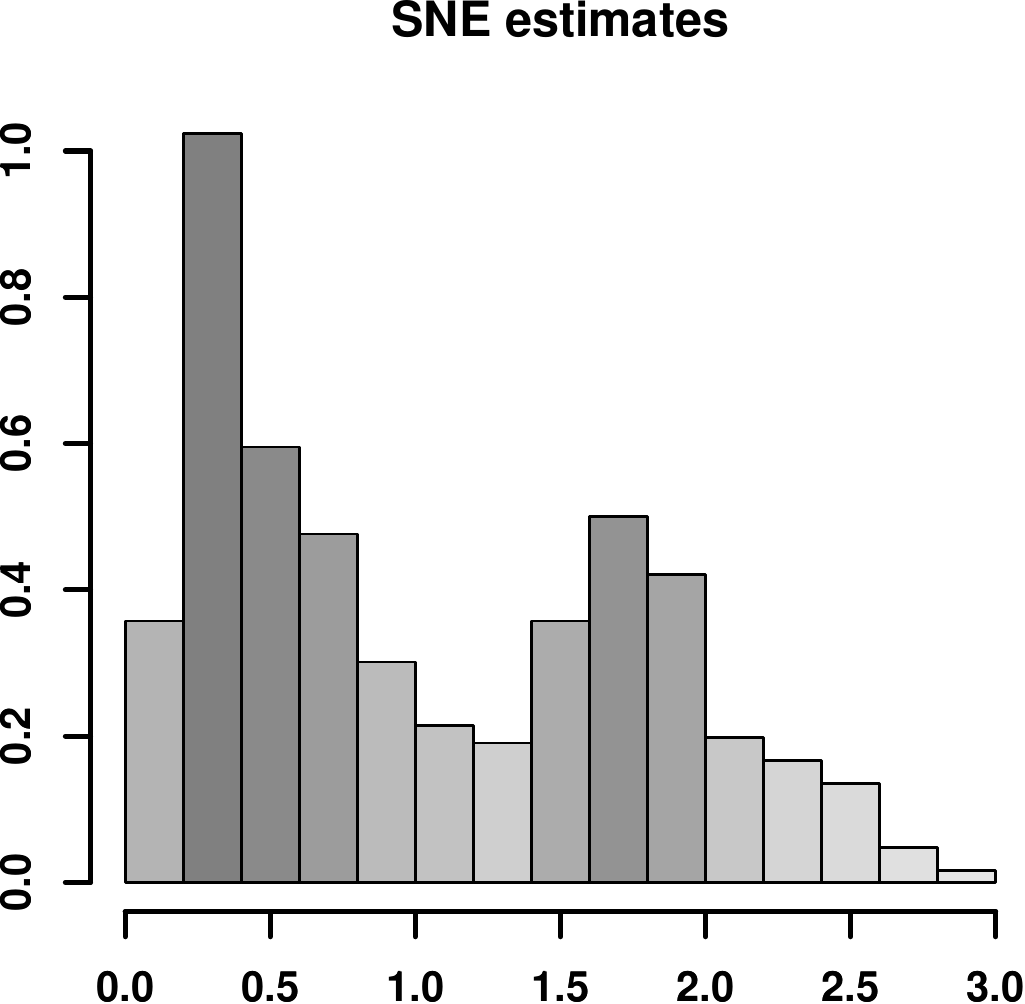} \\
(c) \includegraphics[scale=.34]{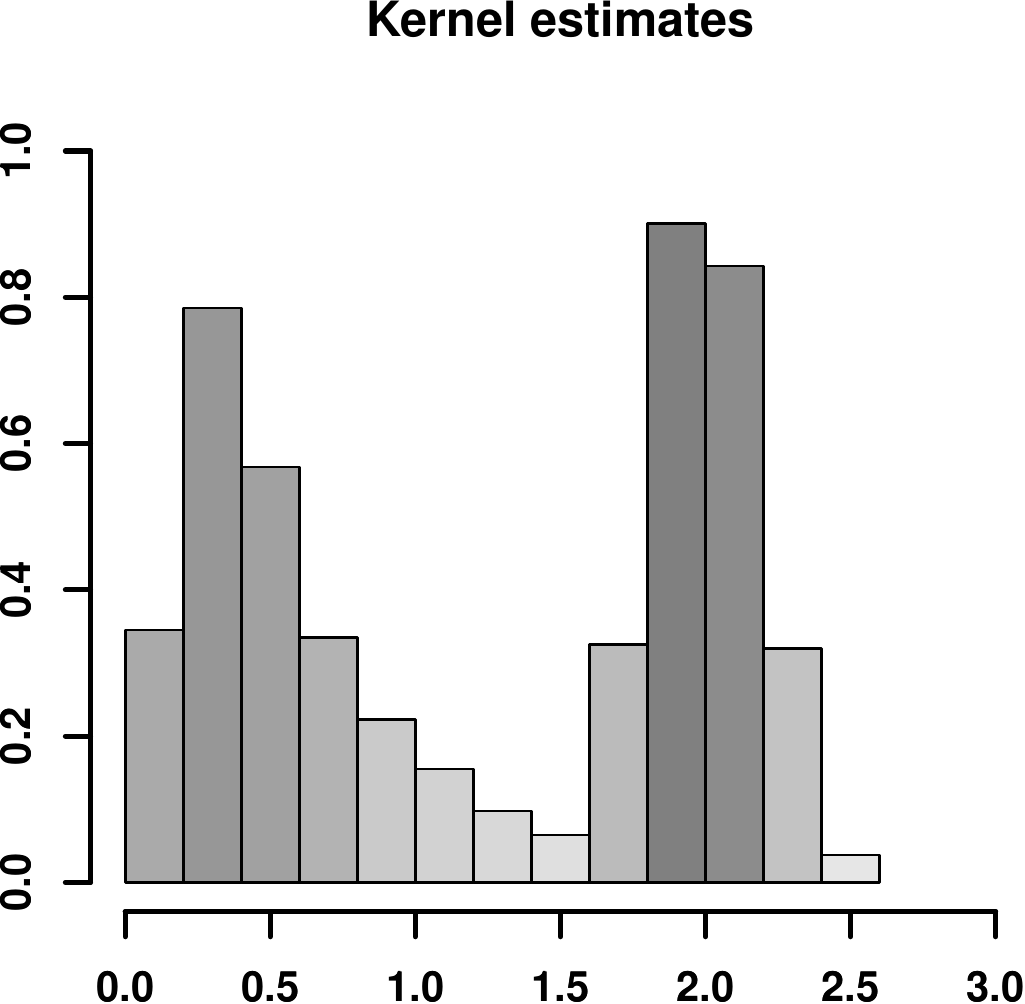}
\end{tabular}
\caption{Comparison of methods for estimating valuations from
  bids. (a) Histogram of true valuations. (b) Valuations estimated
  under the SNE assumption. (c) Density estimation algorithm.}
\label{fig:hist}
\vspace{-.18in}
\end{figure}

\section{CONCLUSION}

We proposed and analyzed two algorithms for learning optimal reserve
prices for generalized second price auctions. Our first algorithm is
based on density estimation and therefore suffers from the standard
problems associated with this family of algorithms. Furthermore, this
algorithm is only well defined when bidders play in equilibrium. Our
second algorithm is novel and is based on learning theory
guarantees. We show that the algorithm admits an efficient $O(n S \log
(n S)) $ implementation. Furthermore, our theoretical guarantees are
more favorable than those presented for the previous algorithm of
\cite{SunZhou}.  Moreover, even though it is necessary for advertisers
to play in equilibrium for our algorithm to converge to optimality,
when bidders do not play an equilibrium, our algorithm is still well
defined and minimizes a quantity of interest albeit over a smaller
set. We also presented preliminary experimental results showing the
advantages of our algorithm. To our knowledge, this is the first
attempt to apply learning algorithms to the problem of reserve price
selection in GSP auctions. We believe that the use of learning
algorithms in revenue optimization is crucial and that this work may
preface a rich research agenda including extensions of this work to
a general learning setup where auctions and advertisers are
represented by features. Additionally, in our analysis, we considered
two different ranking rules. It would be interesting to combine the
algorithm of \cite{ZhuWang} with this work to learn both a ranking
rule and an optimal reserve price. Finally, we
provided the first analysis of convergence of bidding functions in an
empirical equilibrium to the true bidding function. This result on its
own is of great importance as it justifies the common assumption of
advertisers playing in a Bayes-Nash equilibrium.

\newpage
\bibliographystyle{chicago}
\bibliography{gsp}

\newpage
\appendix

\section{THE ENVELOPE THEOREM}

The envelope theorem is a well known result in auction mechanism design
characterizing the maximum of a parametrized family of functions. The
most general form of this theorem is due to \cite{Milgromenvelope} and
we include its proof here for completeness. We will let $X$ be an
arbitrary space  will consider a function $f \colon X \times
[0,1] \to \Rset$ we define the envelope function $V$ and the set
valued function $X^*$ as
\begin{align*}
 V(t) &= \sup_{x \in X} f(x, t) \qquad \text{and} \\
 X^*(t) &= \{x \in X | f(x,t) =  V(t) \}. 
\end{align*}
We show a plot of the envelope function in figure \ref{fig:envelope}.

\begin{theorem}[Envelope Theorem]
\label{th:envelope}
Let $f$ be an absolutely continuous function for every $x \in X$. Suppose
also that there exists an integrable function $b \colon [0,1] \to
\Rset_+$ such that for every $x \in X$, $\frac{df}{dt}(x, t) \leq b(t)$ almost
everywhere in $t$. Then $V$ is absolutely continuous. If in addition
$f(x,\cdot)$ is differentiable for all $x \in X$, $X^*(t) \neq
\emptyset$ almost everywhere on $[0,1]$ and $x^*(t)$ denotes an
arbitrary element in $X^*(t)$, then 
\begin{equation*}
  V(t) = V(0) + \int_0^t \frac{df}{dt}(x^*(s), s) ds.
\end{equation*}
\end{theorem}
\begin{proof}
By definition of $V$, for any $t', t'' \in [0,1]$ we have
\begin{align*}
  |V(t'') - V(t')| & \leq \sup_{x \in X} |f(x, t'') - f(x, t')| \\
& = \sup_{x \in X} \Big| \int_{t'}^{t''} \frac{df}{dt}(x, s) \Big|
\leq \int_{t'}^{t''} b(t) dt.
\end{align*}
This easily implies that $V(t)$ is absolutely
continuous. Therefore, $V$ is differentiable almost everywhere and
$V(t) = V(0) + \int_0^t V'(s) ds$. Finally, if $f(x, t)$ is
differentiable in $t$ then we know that $V'(t) = \frac{df}{dt}(x^*(t),
t)$ for any $x^*(t) \in X^*(t)$ whenever $V'(t)$ exists and the result follows.
\end{proof}

\begin{figure}[t]
\centering
\includegraphics[scale=.5]{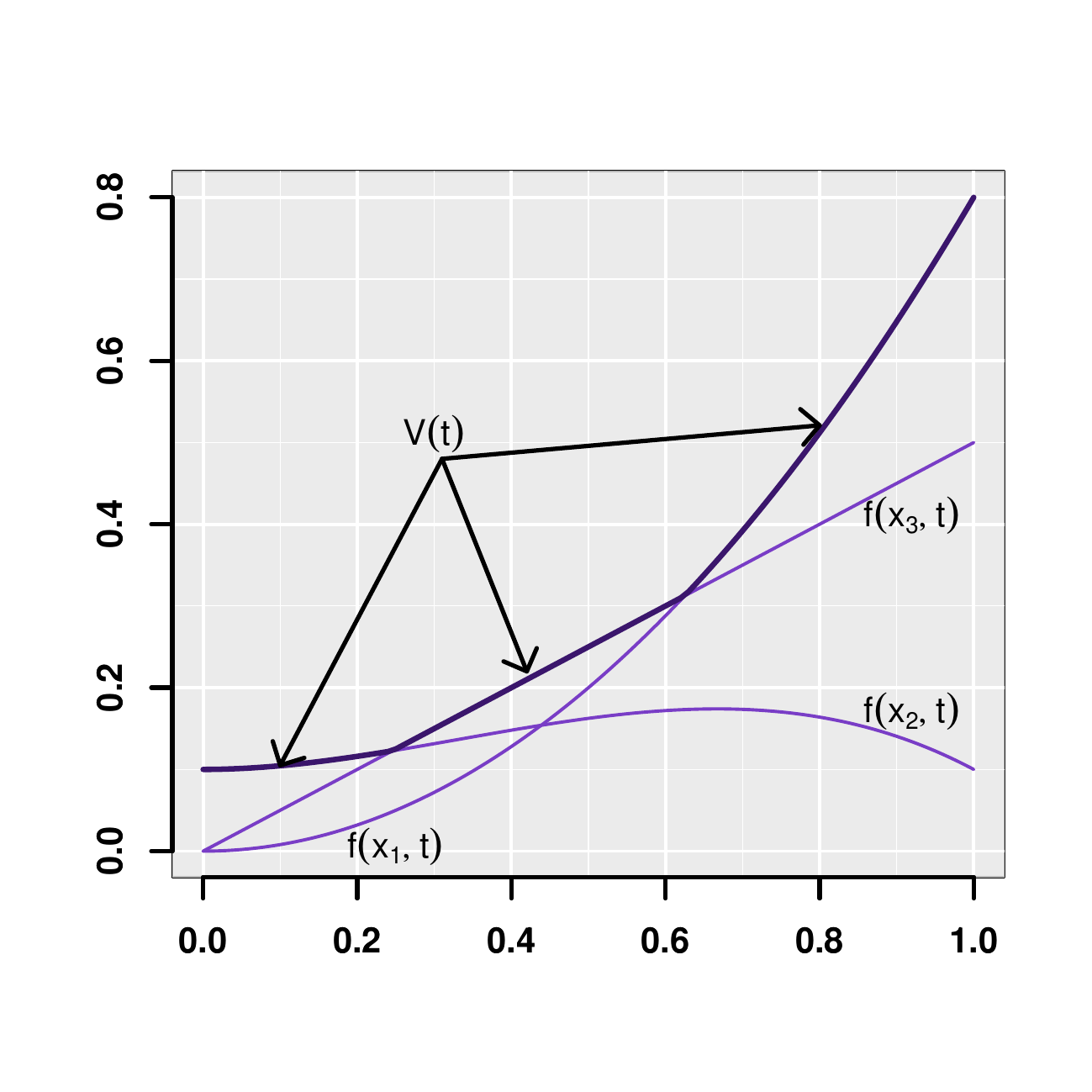}
\caption{Illustration of the envelope function.}
\label{fig:envelope}
\end{figure}

\section{ELEMENTARY CALCULATIONS}

We present elementary results of Calculus that will be used throughout
the rest of this Appendix.
\label{sec:elementary}
\begin{lemma}
\label{lemma:derivative}
The following equality holds for any $k \in \mathbb{N}$:
\begin{equation*}
  \Delta \F_i^k  = \frac{k}{n} \F_{i-1}^{k-1} +
  \frac{i^{k-2}}{n^{k-2}} O \Big(\frac{1}{n^2} \Big), 
\end{equation*}
and
\begin{equation*}
  \Delta \G_i^k = -\frac{k}{n} \G_{i-1}^{k-1} + O\Big(\frac{1}{n^2} \Big).
\end{equation*}
\end{lemma}
\begin{proof}
The result follows from a straightforward application of Taylor's
theorem to the function $h(x) = x^k$. Notice that $\F^k_i = h(i/n)$,
therefore:
\begin{align*}
 \Delta \F_i^k 
& = h \Big( \frac{i-1}{n} + \frac{1}{n}  \Big) 
- h \Big( \frac{i-1}{n} \Big) \\
& = h'\Big(\frac{i-1}{n} \Big) \frac{1}{n}
+ h''(\zeta_i ) \frac{1}{2 n^2} \\
& = \frac{k}{n} \F_{i-1}^{k-1} + h''(\zeta_i) \frac{1}{2 n^2},
\end{align*}
for some  $\zeta_i \in [(i-1)/n, i/n]$. Since $h''(x) = k (k-1)
x^{k-2}$, it follows that the last term in the previous expression is
in $(i/n)^{k-2}O(1/n^2)$. The second equality can be similarly proved.
\end{proof}

\begin{proposition}
\label{prop:binom}
Let $a, b \in \Rset$ and $N \geq 1$ be an integer, then
\begin{equation}
\label{eq:binom}
  \sum_{j=0}^N \binom{N}{j} \frac{a^j b^{N-j}}{j + 1}  = \frac{(a +
    b)^{N+1} - b^{N+1}}{a (N+1)} 
\end{equation}
\end{proposition}
\begin{proof}
The proof relies on the fact that $\frac{a^j}{j + 1} = \frac{1}{a}
\int_0^a t^j dt$. The left hand side of \eqref{eq:binom} is then  equal to
\begin{align*}
\frac{1}{a} \int_0^a \sum_{j=0}^N\binom{N}{j} t^j b^{N-j} dt 
& = \frac{1}{a} \int_0^a (t + b)^N dt \\
&= \frac{(a + b)^{N+1} -  b^{N+1}}{a (N+1)}.
\end{align*}
\end{proof}

\begin{lemma}
\label{lemma:gromwall}
If the sequence $a_i \geq 0 $ satisfies
\begin{align*}
  a_i \leq \delta & \quad \forall i \leq r \\
  a_i \leq A + B \sum_{j=1}^{i-1} a_j & \quad \forall i > r. 
\end{align*}
Then $a_i \leq ( A + r \delta B)( 1 + B)^{i -r  -1}  \leq (A + r
\delta B) e^{B (i - r - 1)} \; \forall i > r$. 
\end{lemma}
This lemma is well known in the numerical analysis community and we
include the proof here for completeness. 
\begin{proof}
  We proceed by induction on $i$. The base of our induction is given
  by $i = r+1$ and it can be trivially verified. Indeed,  by assumption
 \begin{equation*}
 a_{r+1} \leq A + r \delta B .  
 \end{equation*}
Let us assume that the proposition holds for values less than $i$ and
let us try to show it also holds for $i$. 
\begin{align*}
  a_i &\leq A + B \sum_{j=1}^r a_j + B \sum_{j=r+1}^{i-1} a_j \\
& \leq A + r B \delta + B \sum_{j=r+1}^{i-1} (A + r B \delta)(1 + B)^{j - r
  -1}  \\
& = A + r B \delta + (A + r B \delta) B \sum_{j=0}^{i - r- 2} (1 +
B)^j\\
& = A + r B \delta + (A + rB \delta ) B \frac{(1 + B)^{i-r-1} -
  1}{B} \\
& = (A+ r B \delta)( 1+ B)^{i-r-1}.
\end{align*}
\end{proof}

\begin{lemma}
\label{lemma:lambert}
Let $W_0 \colon [e, \infty) \to \Rset$ denote the main branch of the
Lambert function, i.e. $W_0(x) e^{W_0(x)} = x$. The following
inequality holds for every $x \in [e, \infty)$. 
\begin{equation*}
 \log(x) \geq W_0(x).
\end{equation*}
By definition of $W_0$ we see that $W_0(e) = 1$. Moreover, $W_0$ is an
increasing function. Therefore for any $x\in [e, \infty)$
\begin{align*}
&  W_0(x) \geq 1 \\
\Rightarrow & W_0(x) x \geq x \\
\Rightarrow & W_0(x) x \geq W_0(x) e^{W_0(x)} \\
\Rightarrow & x \geq e^{W_0(x)}.
\end{align*}
The result follows by taking logarithms on both sides of the last
inequality. 
\end{lemma}

\section{PROOF OF PROPOSITION~\ref{prop:linear}}
\label{sec:prop-linear}

Here, we derive the linear equation that must be satisfied by the
bidding function $\h \beta$. For the most part, we adapt the analysis of
\cite{Gomes} to a discrete setting.

\begin{repproposition}{prop:empz}
In a symmetric efficient equilibrium of the discrete GSP, the
probability $\h z_s(v)$ that an advertiser with valuation $v$ is
assigned to slot $s$ is given by
\begin{multline*}
 \h z_s(v) \\
= \sum_{j=0}^{N-s} \sum_{k=0}^{s-1} \binom{N-1}{j,k,
 N \! - \! 1 \! -\! j\!- \!k}\frac{\F_{i-1}^j \G_i^k}{(N-j-k)n^{N-1-j-k}},
\end{multline*}
if $v = v_i$ and otherwise by
\begin{equation*}
  \h z_s(v) = \binom{N-1}{s-1} \lim_{v' \rightarrow v^-} \h F(v')^{p} (1 - \h
  F(v))^{s-1} =: \h z_s^-(v),
\end{equation*}
where $p = N - s$.
\end{repproposition}
\begin{proof}
Since advertisers play an efficient equilibrium, these probabilities
depend only on the advertisers' valuations. Let $A_{j,k}(s, v)$
denote the event that $j$ buyers have a valuation lower than $v$, $k$ of
them have a valuation higher than $v$ and $N - 1 - j -k$ a
valuation exactly equal to $v$. Then, the probability of assigning
$s$ to an advertiser with value $v$ is given by
\begin{equation}
\label{eq:sumAij}
 \sum_{j=0}^{N-s}\sum_{k=0}^{s-1}  \frac{1}{N-i-j} \Pr(A_{j,k}(s,v)).
\end{equation}
The factor $\frac{1}{N - i -j}$ appears due to the fact that the slot
is randomly assigned in the case of a tie. When $v = v_i$, this
probability is easily seen to be:
\begin{equation*}
 \binom{N-1}{j,k,  N \! - \! 1 \! -\! j\!- \!k}
 \frac{\F_{i-1}^j \G_i^k}{n^{N-1-j-k}}.
\end{equation*}
On the other hand, if $v \in (v_{i-1}, v_i)$ the event $A_{j,k}(s, v)$
happens with probability zero unless $j = N-s$ and $k = s-1$. Therefore,
\eqref{eq:sumAij} simplifies to
\begin{multline*}
\coeff \h F(v)^{p} (1 - \h F(v))^{s-1} \\
= \coeff \lim_{v' \rightarrow v^-} \h F(v')^{p} (1 - \h  F(v))^{s-1} .
\end{multline*}
\end{proof}
\begin{proposition}
\label{prop:paymenteq}
Let $\E[P^{PE}(v)]$ denote the expected payoff of an advertiser with value
$v$ at equilibrium. Then 
\begin{equation*}
  \E[P^{PE}(v_i)] = \sum_{s=1}^S c_s\Big[ \h z_s( v_i) v_i -
  \sum_{j=1} \h z_s^-(v_i) (v_i - v_{i-1}) \Big].
\end{equation*}
\end{proposition}
\begin{proof}
By the revelation principle \citep{gibbons1992}, there exists a truth
revealing mechanism with the same expected payoff function as the GSP
with bidders playing an equilibrium. For this mechanism, we then must have
\begin{equation*}
  v \in \arg\!\max_{\overline v \in [0,1]} \sum_{s=1}^S c_s \h
  z_s(\overline v) v - \E[P^{PE}(v)].
\end{equation*}
By the envelope theorem (see Theorem~\ref{th:envelope}), we have
\begin{equation*}
 \sum_{s=1}^S c_s \h z_s(v_i) v_i  - \E[P^{PE}(v_i)] 
= -\E[P^{PE}(0)] +  \sum_{s=1}^S \int_0^{v_i} \h z_s(t) dt. 
\end{equation*}
Since the expected payoff of an advertiser with valuation $0$ should
be zero too, we see that
\begin{equation*}
 \E[P^{PE}(v_i)] = c_s \h z_s(v_i) v_i - \int_0^{v_i} \h z_s(t) dt. 
\end{equation*}
Using the fact that $\h z_s(t) \equiv \h z^-_s(v_i)$ for $t \in (v_{i-1},
v_i)$ we obtain the desired expression. 
\end{proof}
\begin{repproposition}{prop:linear}
If the discrete GSP auction admits a symmetric efficient
equilibrium, then its bidding function $\h \beta$ satisfies $\h
\beta(v_i) = \bbeta_i$, where $\bbeta$ is the solution of
the following linear equation:
\begin{equation*}
 \M \bbeta = \mat u.
\end{equation*}
where $\M = \sum_{s=1}^S c_s \M(s)$ and
\begin{equation*}
\mat u_i = \sum_{s=1}^S \Big( c_s z_s(v_i)v_i - \sum_{j=1}^i \h
z_s^-(v_j) \Delta \mat v_j \Big).
\end{equation*}
\end{repproposition}
\begin{proof}
  Let $\E[P^{\h \beta}(v_i)]$ denote the expected payoff of an
advertiser with value $v_i$ when all advertisers play the bidding
function $\h \beta$. Let $A(s, v_i, v_j)$ denote the event that an
advertiser with value $v_i$ gets assigned slot $s$ and the $s$-th
highest valuation among the remaining $N-1$ advertisers is $v_j$. If
the event $A(s, v_i, v_j)$ takes place, then the advertiser has a
expected payoff of $c_s \beta(v_j)$. Thus,
\begin{equation*}
  \E[P^\beta(v_i)] =  \sum_{s=1}^S c_s \sum_{j=1}^i  \beta(v_j)
  \Pr(A(s,v_i, v_j)).
\end{equation*}
In order for event $A(s, v_i, v_j)$ to occur for $i \neq j$, $N-s$
advertisers must have valuations less than or equal to $v_j$ with
equality holding for at least one advertiser. Also, the valuation of
$s-1$ advertisers must be greater than or equal to $v_i$. Keeping in mind
that a slot is assigned randomly in the event of a tie, we see that
$A(s, v_i, v_j)$ occurs with probability
\begin{align*}
\lefteqn{\sum_{l=0}^{N-s-1} \sum_{k=0}^{s-1} \binom{N\!-\!1}{s\!-\!1}
  \binom{s\!-\!1}{k} \binom{N\!-\!s}{l} \frac{\F_{j-1}^{l}}{n^{N-s-l}}
 \frac{\G_i^{s-1-k}}{(k+1) n^{k}}} \\
& =\coeff \sum_{l=0}^{N-s-1} \binom{N\!-\!s}{l} \frac{\F_{j-1}^l}{n^{N-s-l}}
\sum_{k=0}^{s-1} \binom{s\!-\!1}{k} \frac{\G_i^{s-1-k}}{(k+1) n^k} \\
& = \coeff \Big(\big(\F_{j-1} + \frac{1}{n}\big)^{N-s} - F_{j-1}^{N-s}
\Big) \Big(\frac{n \big(\G_{i-1}^{s} - \G_i^{s} \big)}{s} \Big) \\
&= -\binom{N-1}{s-1} \frac{n \Delta \F_j  \Delta \G_i}{s},
\end{align*}
where the second equality follows from an application of the binomial
theorem and Proposition~\ref{prop:binom}. On the other hand if $i = j$
this probability is given by:
\begin{equation*}
\sum_{j=0}^{N-s-1} \sum_{k=0}^{s-1}
\binom{N-1}{j,k,N-1-j-k}  \frac{\F_{i-1}^j  \G_i^k}{(N-j-k) n^{N-1-j-k}}
\end{equation*}
It is now clear that $\M(s)_{i,j} = \Pr(A(s, v_i, v_j))$ for $i \geq
j$. Finally, given that in equilibrium the equality $\E[P^{PE}(v)] =
\E[P^{\h \beta}(v)]$ must hold, by Proposition~\ref{prop:paymenteq}, we
see that $\bbeta$ must satisfy equation \eqref{eq:gsp-linear}.
\end{proof}
We conclude this section with a simpler expression for
$\M_{ii}(s)$. By adding and subtracting the term $j = N - s$ in the
expression defining $\M_{ii}(s)$ we obtain
\begin{align}
\M_{ii}(s) & =\h z_s(v_i)
-\sum_{k=0}^{s-1} \binom{N \! - \!1}{N \!-\! s, k, s\!-\!1\!-\!k}
 \frac{\F_{i-1}^{p}\G_i^k}{(s-k)n^{s-1-k}} \nonumber \\
&= \h z_s(v_i) - \binom{N-1}{s-1} \sum_{k=1}^{s-1}
\binom{s-1}{k} \frac{\F_{i-1}^{p}\G_i^k}{(s-k)n^{s-1-k}} \nonumber \\
& = \h z_s(v_i) + \binom{N-1}{s-1} \F_{i-1}^{p} \frac{n \Delta
  \G_i}{s}\label{eq:Miizs},
\end{align}
where again we used Proposition~\ref{prop:binom} for the last equality.

\section{HIGH PROBABILITY BOUNDS}
\label{sec:bounds}

In order to improve the readability of our proofs we use a fixed
variable $C$ to refer to a universal constant even though this
constant may be different in different lines of a proof. 
\begin{theorem}(Glivenko-Cantelli Theorem)
\label{th:glivenko}
Let $v_1, \ldots, v_n$ be an i.i.d.\ sample drawn from a distribution
$F$. If $\h F$ denotes the empirical distribution function induced by
this sample, then with probability at least $1- \delta$ for all $v \in
\Rset$
\begin{equation*}
  |\h F(v) - F(v) | \leq C \sqrt{\frac{\log(1/\delta)}{n}}.
\end{equation*}
\end{theorem}
\begin{proposition}
\label{prop:hp}
Let $X_1, \ldots, X_n$ be an $i.i.d$ sample from a distribution $F$
supported in $[0,1]$. Suppose $F$ admits a density $f$ and assume
$f(x) > c$ for all $x \in [0,1]$. If $X^{(1)}, \ldots , X^{(n)}$
denote the order statistics of a sample of size $n$ and we let
$X^{(0)} = 0$, then
\begin{equation*}
\Pr(\max_{i \in \{1, \ldots, n\}} X^{(i)} - X^{(i-1)} > \epsilon )
 \leq \frac{3}{\e} e^{-c\e n/2}. 
\end{equation*}
In particular, with probability at least $1 - \delta$:
\begin{equation}
  \max_{i \in \{1, \ldots, n\}}  X^{(i)} - X^{(i-1)} \leq \frac{1}{n} q(n, \delta),
\end{equation}
where $q(n, \delta) = \frac{2}{c} \log \Big( \frac{n c}{2 \delta} \Big)$.
\end{proposition}
\begin{proof}
Divide the interval $[0,1]$ into $k \leq \lceil 2/ \epsilon \rceil$
sub-intervals of size $\frac{\e}{2}$. Denote this sub-intervals by
$I_1, \ldots, I_k$, with $I_j = [a_j,b_j]$ . If there exists $i$ such
that $X^{(i)} - X^{(i-1)} > \e$ then at least one of these sub-intervals
must not contain any samples. Therefore:
\begin{align*}
\Pr(\max_{i \in \{1, \ldots, n\}} X^{(i)} - X^{(i-1)} >  \epsilon ) 
& \leq \Pr (\exists \;\; j \; \text{s.t}\;\;  X_i \notin I_j \; \forall i) \\
& \leq \sum_{j=1}^{\lceil 2 / \e \rceil } \Pr(X_i \notin I_j \; \forall
i).
\end{align*}
Using the fact that the sample is i.i.d.\ and that $F(b_k)
-F(a_k) \geq \min_{x \in [a_k, b_k] } f(x)(b_k - a_k) \geq c(b_k -
a_k)$, we may bound the last term by
\begin{align*}
\Big(\frac{2  + \e }{\e} \Big)  (1 - (F(b_k) - F(a_k)))^n 
& \leq \frac{3}{\epsilon} (1 - c(b_k - a_k))^n \\
& \leq \frac{3}{\epsilon} e^{-c \e n/2}.   
\end{align*}
The equation $\frac{3}{\e} e^{-c \e n/2} = \delta$ implies
$\e =
\frac{2}{nc} W_0(\frac{3 n c}{2 \delta})$, where $W_0$ denotes the main
branch of the Lambert function (the inverse of the function $x \mapsto
x e^x$). By Lemma~\ref{lemma:lambert}, for  $x \in [e, \infty)$ we
have 
\begin{equation}
\label{eq:lambert}
   \log(x) \geq W_0(x). 
\end{equation}
Therefore,  with probability at least $1 - \delta$
\begin{equation*}
\max_{i \in \{1, \ldots, n\}} X^{(i)} - X^{(i-1)}
\leq  \frac{2}{nc} \log \Big(\frac{3 c n}{2 \delta} \Big). 
\end{equation*}
\end{proof}
The following estimates will be used in the proof of
Theorem~\ref{th:convergence}.
\begin{lemma}
\label{lemma:diffp}
Let $p \geq 1$ be an integer. If $i > \sqrt{n}$ , then for any
$t \in [v_{i-1}, v_i]$ the following inequality is satisfied with
probability at least $1 - \delta$:
\begin{equation*}
|F^p(v) - \F_{i-1}^p| \leq C \frac{i^{p-1}}{n^{p-1}}
\frac{\log(2/\delta)^{\frac{p-1}{2}}}{\sqrt{n}} q(n, 2/\delta)
\end{equation*}
\end{lemma}
\begin{proof}
The left hand side of the above inequality may be decomposed as
\begin{flalign*}
&|F^p(v) - \F^p_{i-1}|&\\
& \leq |F^p(v) - F^p(v_{i-1}) |  + |F^p(v_{i-1}) -  \F^p_{i-1}| \\
& \leq p |F(\zeta_i)^{p-1} f(\zeta_i)| (v_i - v_{i-1})
+ p \F_{i-1}^{p-1} (F(v_{i-1}) - \F_{i-1}) \\
& \leq C \frac{q(n, \frac{2}{\delta})}{n} F(\zeta_i)^{p-1} + C
\frac{i^{p-1}}{n^{p-1}} \sqrt{\frac{\log(2/\delta)}{n}},
\end{flalign*}
for some $\zeta_i \in (v_{i-1}, v_i)$. The second inequality follows
from Taylor's theorem and we have used Glivenko-Cantelli's theorem and
Proposition~\ref{prop:hp} for the last inequality. Moreover, we know
$F(v_i) \leq \F_i + \sqrt{\frac{\log 2/\delta}{n}} \leq C
\frac{\sqrt{\log 2/\delta}(i + \sqrt{n})}{n}$. Finally, since $i \geq
\sqrt{n}$ it follows that
\begin{equation*}
  F(\zeta_i)^{p-1} \leq F(v_i)^{p-1} \leq C\Big( \frac{i^{p-1}}{n^{p-1}}
  \log(2/\delta)^{(p-1)/2}\Big).
\end{equation*}
Replacing this term in our original bound yields the result.
\end{proof}

\begin{proposition}
\label{prop:intbounds}
Let $\psi : [0,1] \to \Rset$ be a twice continuously differentiable
function. With probability at least $1 - \delta$ the following bound
holds for all $i > \sqrt{n}$
\begin{multline*}
\Big| \int_0^{v_i} F^{p}(t) dt
- \sum_{j =1}^{i-1} \F_{j-1}^{p} \Delta \mat v_j \Big | \\
\leq C \frac{i^{p}}{n^{p}} \frac{\log(2/\delta)^{p/2}}{\sqrt{n}}
q(n, \delta/2)^2.
\end{multline*}
and 
\begin{multline*}
\Big| \int_0^{v_i} \psi(t) p F^{p-1}(t) f(t) dt -
\sum_{j=1}^{i-2} \psi(v_j) \Delta F^{p}_j \Big| \\
\leq C \frac{i^{p}}{n^{p}} \frac{\log(2/\delta)^{p/2}}{\sqrt{n}}
q(n, \delta/2)^2.
\end{multline*}
\end{proposition}
\begin{proof}
By splitting the integral along the intervals $[v_{j-1}, v_j]$ 
 we obtain
 \begin{multline}
\label{eq:intbounddiff}
\Big| \int_0^{v_i} F^{p}(t) dt -\sum_{j =1}^{i-1}  \F_{j-1}^{p} \Delta \mat v_j \Big | \\
\leq \Big | \sum_{j=1}^{i-1} \int_{v_{j-1}}^{v_j} F^{p}(t) - \F_{j-1}^{p}dt
\Big| + F^p(v_i) (v_i - v_{i-1}) 
\end{multline}
By Lemma~\ref{lemma:diffp}, for $t \in [v_{j-1}, v_j]$ we have:
\begin{equation*}
|F^p(t) - \F_{j-1}^p| \leq
C\frac{j^{p-1}}{n^{p-1}}\frac{\log(2/\delta)^{\frac{p-1}{2}}}{\sqrt{n}} q(n, \delta/2).
\end{equation*}
Using the same argument of Lemma~\ref{lemma:diffp} we see that for $i
\geq \sqrt{n}$ 
\begin{equation*}
 F(v_i)^p \leq C \Big(\frac{ i \sqrt{\log(2 /\delta)}}{n}\Big)^p
\end{equation*}
Therefore we may bound \eqref{eq:intbounddiff} by
\begin{multline*}
C \frac{i^{p-1}}{n^{p-1}} \frac{\log(2/\delta)^{\frac{p-1}{2}}}{\sqrt{n}} \Big( q(n,
\delta/2) \sum_{j=1}^{i-1} v_j \\ 
+  \frac{i (v_i - v_{i-1})\sqrt{\log(2/\delta)}}{n}\Big).
\end{multline*}
We can again use Proposition~\ref{prop:hp} to bound the sum by
$\frac{i}{n}q(n, \delta/2)$ and the result follows. In order to proof
the second bound we first do integration by parts to obtain
\begin{equation*}
  \int_0^{v_i} \psi(t) p F^{p-1} f(t) dt = \psi(v_i) F^{p}(v_i) - \int_0^{v_i}
  \psi'(t) F^p(t) dt. 
\end{equation*}
Similarly 
\begin{equation*}
 \sum_{j=1}^{i-2}\psi(v_j) \Delta \F_j^p = \psi(v_{i-2}) \F^p_{i-2} -
 \sum_{j=1}^{i-2} \F_j^p \big(\psi(v_j) - \psi(v_{j-1})\big).
\end{equation*}
Using the fact that $\psi$ is twice continuously differentiable, we can
recover the desired bound by following similar steps as before. 
\end{proof}

\begin{proposition}
\label{prop:deltadiff}
With probability at least $1 - \delta$ the following inequality holds
for all $i$
\begin{equation*}
\label{eq:deltadiff}
 \Big|(s-1) G(v_i)^{s-2} - n^2 \frac{\Delta^2\G_i^s}{s} \Big| \\
 \leq C \sqrt{\frac{\log(1/\delta)}{n}}.
\end{equation*}
\end{proposition}
\begin{proof}
By Lemma~\ref{lemma:derivative} we know that 
\begin{equation*}
 n^2 \frac{\Delta^2 \G_i^s}{s} = (s -1)\G_i^{s-2} + O\Big(\frac{1}{n} \Big) 
\end{equation*}
Therefore the left hand side of \eqref{eq:deltadiff} can be bounded by
\begin{equation*}
  (s-1)|G(v_i)^{s-2} - \G_i^{s-2}| + \frac{C}{n}. 
\end{equation*}
The result now follows from Glivenko-Cantelli's theorem.
\end{proof}

\begin{proposition}
\label{prop:Miiapprox}
With probability at least $1 - \delta$ the following bound holds for
all $i$
\begin{multline*}
 \Big| \binom{N-1}{s-1}p G(v_i)^{s-1} F(v_i)^{p} -  2 n \M_{ii}(s) \Big| \\ 
\leq C \frac{i^{p-2}}{n^{p-2}}
\frac{(\log(2/\delta))^{\frac{p-2}{2}}}{\sqrt{n}} q(n,\delta/2).
\end{multline*}
\end{proposition}
\begin{proof}
By analyzing the sum defining $\M_{ii}(s)$ we see that all terms with
exception of the term given by $j = N-s-1$ and $k = s-1$ have a factor
of $\frac{i^{p-2}}{n^{p-2}}\frac{1}{n^2}$. Therefore, 
\begin{equation}
\label{eq:Miiapprox}
\M_{ii}(s) = \frac{1}{2n} \binom{N\!-\!1}{s\!-\!1} p \F_{i-1}^{p-1} \G_i^{s-1} 
+ \frac{i^{p-2}}{n^{p-2}} O\Big( \frac{1}{n^2} \Big).
\end{equation}
Furthermore, by Theorem~\ref{th:glivenko} we have
\begin{equation}
\label{eq:Qdiff}
|\G_i^{s-1} - G(v_i)^{s-1}| \leq C \sqrt{\frac{\log(2/\delta)}{n}}.
\end{equation}
Similarly, by Lemma~\ref{lemma:diffp}
\begin{equation}
  \label{eq:Gdiff}
|\F_{i-1}^{p-1} - F(v_i)^{p-1} | C \leq  \frac{i^{p-2}}{n^{p-2}}
\frac{(\log(2/\delta))^{\frac{p-2}{2}}}{\sqrt{n}} q(n,\delta/2).
\end{equation}
From equation \eqref{eq:Miiapprox} and inequalities \eqref{eq:Qdiff}
and \eqref{eq:Gdiff} we can thus infer that 
\begin{align*}
\lefteqn{\big| pG(v_i)^{s-1} F(v_i)^{p}  2 n \M_{ii}(s) \big|} \\
&\leq C \Big( p \F_{i-1}^{p-1} | G(v_i)^{s-1} \!\!-\!\! \G_i^{s-1}| 
\!+\! G(v_i)^{s-1} p| F(v_i)^{p-1} \!\!-\!\! \F_{i-1}^{p-1} | \Big)  \\
& \qquad+ C \frac{i^{p-2}}{n^{p-2}} \frac{1}{n^2} \\
& \leq C \frac{i^{p-2}}{n^{p-2}} \Big(\frac{i}{n}
\sqrt{\frac{\log(2/\delta)}{n}}
+ \frac{(\log(2/\delta))^{\frac{p-2}{2}}}{\sqrt{n}} q(n,\delta/2)
+ \frac{1}{n^2} \Big)  
\end{align*}
The desired bound follows trivially from the last inequality. 
\end{proof}

\section{SOLUTION PROPERTIES}
\label{sec:properties}

A standard way to solve a Volterra equation of the first kind is
to differentiate the equation and transform it into an equation of the
second kind. As mentioned before this may only be done if the kernel
defining the equation satisfies $K(t,t) \geq c > 0$ for all $t$. Here
we take the discrete derivative of \eqref{eq:gsp-linear} and show that in
spite of the fact that the new system remains ill-conditioned the
solution of this equation has a particular property that allows us to
show the solution $\bbeta$ will be close to the solution
$\overline{\bbeta}$ of surrogate linear system which, in turn,  will
also be close to the true bidding function $\beta$. 
\begin{proposition}
\label{prop:dM}
The solution $\bbeta$ of equation \eqref{eq:gsp-linear} also satisfies the
following equation
\begin{equation}
\label{eq:difflinear}
  d \M \bbeta = d \mat u 
\end{equation}
where $ d \M_{ij} = \M_{i,j} - \M_{i-1,j}$ and $d \mat u_i = \mat u_i -
\mat u_{i-1}$. Furthermore, for $j \leq i-2$ 
\begin{equation*}
  d \M_{ij} = -\sum_{s=1}^S c_s\coeff \frac{n \Delta \F_j \Delta^2 \G_i^s}{s}
\end{equation*}
and
\begin{equation*}
 d \mat u_i = \sum_{i=1}^S c_s \big( v_i \big( \h z_s(v_i) - \h z_s^-(v_i)
 \big) + v_{i-1} \big( \h z_s^-(v_i) - \h z_s(v_{i-1})\big) \big).
\end{equation*}
\begin{proof}
It is clear that the new equation is obtained from \eqref{eq:gsp-linear}
by subtracting row $i -1 $ from row $i$. Therefore $\bbeta$ must also
satisfy this equation. The expression for $d \M_{ij}$ follows directly
from the definition of $\M_{ij}$. Finally,
\begin{align*}
&\h z_s(v_i) v_i -  \sum_{j=1}^i \h z_s^-(v_j) (v_j - v_{j-1}) \\
& \mspace{20mu}-\Big(\h z_s(v_{i-1}) v_{i-1} - \sum_{j=1}^{i-1} \h z_s^-(v_j) (v_j - v_{j-1}) \Big) \\ 
& = v_i \big( \h z_s(v_i) - \h z_s^-(v_i) \big) \\
& \mspace{20mu} + \h z^-_s(v_i)v_{i} - \h z_s(v_{i-1}) v_{i-1} - \h z_s^-(v_i) (v_i - v_{i-1}). 
\end{align*}
Simplifying terms and summing over $s$ yields the desired expression
for $d \mat u_i$.
\end{proof}
\end{proposition}
A straightforward bound on the difference $|\bbeta_i - \beta(v_i)|$
can be obtain by bounding the following quantity:
difference
\begin{equation}
\label{eq:volterratechnique}
 \sum_{j=1}^i d \M_{i,j} (\beta(v_i) - \bbeta_i) = \sum_{j=1}^i d
 \M_{i,j} \beta(v_i) - d \u_i,
\end{equation}
and by then solving the system of inequalities defining $\e_i =
|\beta(v_i) - \bbeta_i|$. In order to do this, however, it is always
assumed that the diagonal terms of the matrix satisfy $\min_{i} n d
\M_{ii} \geq c > 0$ for all $n$, which in view of \eqref{eq:Miiapprox}
does not hold in our case. We therefore must resort to a different
approach. We will first show that for values of $i \leq n^{3/4}$ the
values of $\bbeta_i$ are close to $v_i$ and similarly $\beta(v_i)$
will be close to $v_i$. Therefore for $i \leq n^{3/4}$ we can show
that the difference $|\beta(v_i) - \bbeta_i|$ is small. We will see
that the analysis for $i \gg n^{3/4}$ is in fact more complicated;
yet, by a clever manipulation of the system \eqref{eq:gsp-linear} we
are able to obtain the desired bound. 

\begin{proposition}
\label{prop:MaxMii}
If $c_S > 0$ then there exists a constant $\overline C > 0 $ such that:
\begin{equation*}
  \sum_{s=1}^S c_s \M_{ii}(s) \geq \overline C \Big(\frac{i}{n}
  \Big)^{N - S - 1}  \frac{1}{2 n}
\end{equation*}
\end{proposition}
\begin{proof}
By definition of $\M_{ii}(s)$ it is immediate that
\begin{align*}
  c_s \M_{ii}(s) & \geq \frac{c_s}{2 n} \binom{N-1}{s-1} p \F_{i-1}^{p-1}(\G_i)^{s-1}\\
&=  \frac{1}{2n} C_s \Big(\frac{i-1}{n}\Big)^{p-1} \Big( 1 -
\frac{i}{n} \Big)^{s-1},
\end{align*}
with $C_S = c_s p \coeff$. The sum can thus be lower bounded as follows
\begin{align}
\label{eq:maxMii}
\sum_{s=1}^S c_s \M(s)_{ii} 
& \geq \frac{1}{2n} \max \bigg\{ C_1 \Big(\frac{i-1}{n} \Big) ^{N-2}, \\
& \qquad C_S \Big(\frac{i-1}{n}\Big)^{N-S-1} \Big( 1 - \frac{i}{n} \Big)^{S-1}\bigg\}.\nonumber
\end{align}
When $ C_1 \Big(\frac{i-1}{n} \Big)^{N-2} \geq C_S
\Big(\frac{i-1}{n}\Big)^{N-S-1} \Big( 1 - \frac{i}{n} \Big)^{S-1} $, we
have $ K \frac{i-1}{n} \geq 1 - \frac{i}{n}$, with $K =
(C_1/C_S)^{1/(S-1)}$. Which holds if and only if $ i >
\frac{n+K}{K+1}$. In this case the max term of \eqref{eq:maxMii} is
easily seen to be lower bounded by $C_1(K/K+1)^{N-2}$. On the other
hand, if $i < \frac{n+K}{K+1}$ then we can lower bound this term by
$C_S (K/K+1)^{s-1} \Big(\frac{i}{n} \Big)^{N-S-1}$. The result follows
immediately from these observations.
\end{proof}

\begin{proposition}
\label{prop:subdiag}
For all $i$ and $s$ the following inequality holds:
\begin{equation*}
 | d \M_{ii}(s) -  d \M_{i,i-1}(s) | \leq C \frac{i^{p-2}}{n^{p-2}}\frac{1}{n^2}.
\end{equation*}
\end{proposition}
\begin{proof}
From equation \eqref{eq:Miiapprox} we see that 
\begin{flalign*}
&|d \M_{ii}(s) - d \M_{i,i-1}(s)|& \\
&  = | \M_{ii}(s) + \M_{i-1,i-1}(s) -\M_{i,i-1}(s) | \\ 
& \leq \Big|\M_{ii}(s) - \frac{1}{2} \M_{i,i-1}(s)\Big|
 + \Big| \M_{i-1, i-1}(s) - \frac{1}{2} \M_{i,i-1}(s) \Big| \\
& \leq \coeff \Big(\frac{1}{2} \Big| \frac{p \F_{i-1}^{p-1}
  \G_i^{s-1}}{n} - \frac{\Delta \F_{i-1}^p n \Delta \G_i^s}{s} \Big|
\\
& \mspace{40mu} + \frac{1}{2}\Big| \frac{p \F_{i-2}^{p-1}
  \G_{i-1}^{s-1}}{n} -\frac{n \Delta \F_{i-1}^p \Delta \G_i^s}{s}
\Big| \Big) + C \frac{i^{p-2}}{n^{p-2}} \Big(\frac{1}{n^2} \Big),
\end{flalign*}
A repeated application of Lemma~\ref{lemma:derivative} yields the
desired result. 
\end{proof}
\begin{lemma}
The following holds for every  $s$ and every $i$
\begin{equation*}
\h z_s(v_i) - \h z_s^-(v_i) =  \M_{ii}(s) - \coeff \F_{i-1}^{p} \Big(
\frac{n \Delta \G_i^s}{s} + \G_{i-1}^{s-1}\Big)
\end{equation*}
and
\begin{flalign*}
&\h z_s^-(v_i) - \h z_s(v_{i-1})& \\
 & = \M(s)_{i,i-1} - \M(s)_{i-1, i-1}  - n \coeff \F_{i-2}^{p} \frac{\Delta^2 \G_i^s}{s}\\
 & \mspace{60mu} + \coeff \F_{i-1}^{p} \Big( \G_{i-1}^{s-1} + n \frac{\Delta \G_i^s} {s} \Big) .
\end{flalign*}
\end{lemma}
\begin{proof}
From \eqref{eq:Miizs} we know that
\begin{equation*}
\h z_s(v_i) - \h z_s^-(v_i) = \M_{ii}(s) - \coeff n \F_{i-1}^{p}
\frac{\Delta \G^s_i}{s} -\h z_s^-(v_i).  
\end{equation*}
By using the definition of $\h z_s^-(v_i)$ we can verify that the right
hand side of the above equation is in fact equal to 
\begin{equation*}
  \M_{ii}(s) - \coeff \F_{i-1}^{p} \Big(\frac{n\Delta \G_i^s}{s} 
+ \G_{i-1}^{s-1}\Big).
\end{equation*}
The second statement can be similarly proved
\begin{multline}
\label{eq:zmminz}
  \h z_s^-(v_i) - \h z_s(v_{i-1})  = \h z_s^-(v_i) - \M(s)_{i-1,i-1} \\
 + n \coeff \F_{i-2}^{p}  \frac{\Delta \G_{i-1} ^s}{s} 
 + \M(s)_{i,i-1} - \M(s)_{i,i-1}.
\end{multline}
On the other hand we have 
\begin{flalign*}
&n \coeff \F_{i-2}^{p}  \frac{\Delta \G_{i-1}^s}{s}  - \M(s)_{i,i-1}& \\
& = n \coeff \Big[\F_{i-2}^{p}  \frac{\Delta \G_{i-1}^s}{s} 
+ \frac{(\F_{i-1}^{p} -  \F_{i-2}^{p}) \Delta \G_{i}^s }{s} \Big]\\
& = n \coeff \Big[\F_{i-1}^{p} \frac{ \Delta \G_i^s}{s} -  
\F_{i-2}^{p} \frac{\Delta^2 \G_i^s}{s} \Big]
\end{flalign*}
By replacing this expression into \eqref{eq:zmminz}  and by definition
of $\h z_s^-(v_i)$.
\begin{flalign*}
&\h z_s^-(v_i) - \h z_s(v_{i-1})& \\
 & = \M(s)_{i,i-1} - \M(s)_{i-1, i-1}  - n \coeff \F_{i-2}^{p} \frac{\Delta^2 \G_i^s}{s}\\
 & \mspace{60mu} + \coeff \F_{i-1}^{p} \Big( \G_{i-1}^{s-1} + n \frac{\Delta \G_i^s} {s} \Big) .
\end{flalign*}
\end{proof}

\begin{corollary}
The following equality holds for all $i$ and $s$.
\begin{align*}
d \u_i &= v_i (\h z_s(v_i) - \h z_s^-(v_i)) + v_{i-1} ( \h z_s^-(v_i) - \h
z_s(v_{i-1})) \\
& = v_i d\M_{ii}(s) + v_{i-1} d \M_{i,i-1}(s) + \sum_{j=1}^{i-2}
d\M_{ij}(s) v_j  \\
& - \coeff \frac{n \Delta^2 \G_i^s}{s}\sum_{j=1}^{i-2} \F^{p}_{j-1}
\Delta \mat v_j  \\
&  -(v_i - v_{i-1})\coeff \F_{i-1}^{p} \Big( \G_{i-1}^{s-1} + \frac{n \Delta \G_i^s} {s} \Big) 
\end{align*}
\end{corollary} 
\begin{proof} 
From the previous proposition we know
\begin{align*}
\lefteqn{ v_i (\h z_s(v_i) - \h z_s^-(v_i)) + v_{i-1} ( \h z_s^-(v_i) - \h
z_s(v_{i-1})) } \\
& = v_i \M_{ii}(s) + v_{i-1} (\M(s)_{i,i-1} - \M(s)_{i-1,i-1})  \\
& - v_{i-1} n \coeff \F_{i-2}^{p} \frac{ \Delta^2 \G_i^s}{s}   \\
& + (v_{i-1} - v_i) \coeff \F_{i-1}^{p} \Big( \G_{i-1}^{s-1} +
 \frac{n \Delta \G_i^s} {s} \Big)  \\
& = v_i d\M_{ii}(s) + v_{i-1} d \M_{i,i -1}(s)
- v_{i-1} n \coeff \F_{i-2}^{p} \frac{ \Delta^2 \G_i^s}{s}   \\
& - (v_i - v_{i-1}) \coeff \F_{i-1}^{p} \Big( \G_{i-1}^{s-1} +
 \frac{n \Delta \G_i^s} {s} \Big),
\end{align*}
where the last equality follows from the definition of $d
\M$. Furthermore, by doing summation by parts we see that 
\begin{align*}
\lefteqn{v_{i-1} n \coeff \F_{i-2}^{p} \frac{ \Delta^2 \G_i^s}{s}} \\
& = v_{i-2} \coeff \F_{i-2}^{p} \frac{ n \Delta^2  \G_i^s}{s}  \\
& \mspace{40mu}+ (v_{i-1} - v_{i-2})  \coeff \F_{i-2}^{p} \frac{n \Delta^2 \G_i^s}{s} \\
& = \coeff \frac{n \Delta^2 \G_i^s}{s}  \Big(\sum_{j=1}^{i-2} v_j \Delta
\F^{p}_j  + \sum_{j=1}^{i-2} \F^{p}_{j-1} \Delta \mat v_j \Big) \\
& \mspace{40mu} + (v_{i-1} - v_{i-2})  \coeff \F_{i-2}^{p} \frac{n \Delta^2 \G_i^s}{s} \\
& = - \sum_{j=1}^{i-2} d \M_{ij} v_j + \coeff \frac{n \Delta^2
  \G_i^s}{s}  \sum_{j=1}^{i-1} \F^{p}_{j-1} \Delta \mat v_j, 
\end{align*}
where again we used the definition of $d \M$ in the last equality. By
replacing this expression in the previous chain of equalities we
obtain the desired result.
\end{proof}

\begin{corollary}
Let $\mat p$ denote the vector defined by
\begin{multline*}
\mat p_i 
 =  \sum_{s=1}^S c_s \coeff \frac{n \Delta^2 \G_i^s}{s}\sum_{j=1}^{i-1}
\F^{p}_{j-1} \Delta \mat v_j \\
+ c_s \Delta \mat v_i \coeff \F_{i-1}^{p} \Big( \G_{i-1}^{s-1}
+ \frac{ n \Delta \G_i^s }{s} \Big). 
\end{multline*}
If $\boldsymbol{\psi} = \mat v - \boldsymbol{\beta}$, then 
$\boldsymbol{\psi}$ solves the following system of equations:
\begin{equation}
\label{eq:finalorig}
  d \M \bpsi = \mat p.
\end{equation}
\end{corollary}
\begin{proof}
It is immediate by replacing the expression for $d \u_i$ from the
previous corollary into \eqref{eq:difflinear} and rearranging terms.
\end{proof}
We can now present the main result of this section. 
\begin{proposition}
\label{prop:smallemp}
Under Assumption~\ref{assum:smooth}, with probability at least $1 -
\delta$, the solution $\bpsi$ of equation \eqref{eq:finalorig}
satisfies $\bpsi_i \leq C \frac{i^2}{n^2} q(n, \delta)$.
\end{proposition}
\begin{proof}
By doing forward substitution on equation \eqref{eq:finalorig} we have:
\begin{flalign}
& d \M_{i,i-1} \bpsi_{i-1} + d \M_{ii} \bpsi_i & \nonumber\\
& = \mat p_i + \sum_{j=1}^{i-2} d \M_{ij} \bpsi_j  \nonumber \\
&= \mat p_i + \sum_{s=1}^S c_s \frac{n \Delta^2
\G_i^s}{s} \sum_{j=1}^{i-2} \Delta \F_j^p \bpsi_j.\label{eq:forward}
\end{flalign}
A repeated application of Lemma~\ref{lemma:derivative} shows that 
\begin{equation*}
\mat p_i \leq C \frac{1}{n} \frac{i^{N-S}}{n^{N-S}} \sum_{j=1}^i \Delta \mat v_j,
\end{equation*}
which  by Proposition~\ref{prop:hp} we know it is bounded by
\begin{equation*}
\mat p_i \leq  C \frac{1}{n} \frac{i^{N-S-1}}{n^{N-S-1}} \frac{i^2}{n^2}
q(n, \delta). 
\end{equation*}
Similarly for $j \leq i-2$ we have
\begin{equation*}
 \frac{n \Delta^2 \G_i^s}{s} \Delta \F_j^p \leq
C \frac{1}{n}\frac{i^{N-S-1}}{n^{N-S-1}}\frac{1}{n}.
\end{equation*}
Finally, Assumption~\ref{assum:smooth} implies that $\bpsi \geq 0$ for
all $i$ and since $d \M_{i,i-1} > 0$, the following inequality must
hold for all $i$:
\begin{align*}
 d \M_{ii} \bpsi_i
& \leq d \M_{i,i-1} \bpsi_{i-1} + d \M_{ii} \bpsi_i \\
& \leq C \frac{1}{n} \frac{i^{N-S-1}}{n^{N-S-1}} \Big( \frac{i^2}{n^2}
q(n, \delta) +\frac{1}{n} \sum_{j=1}^{i-2} \psi_j\Big).
\end{align*}
In view of Proposition~\ref{prop:MaxMii} we know that $d \M_{ii} \geq
\overline C \frac{1}{n} \frac{i^{N-S-1}}{n^{N-S-1}}$, therefore after
dividing both sides of the inequality by $d \M_{ii}$, it follows that
\begin{equation*}
\bpsi_i \leq C \frac{i^2}{n^2} q(n, \delta)
+ \frac{1}{n} \sum_{j=1}^{i-2} \psi_j.   
\end{equation*}
Applying Lemma~\ref{lemma:gromwall} with $A = C \frac{i^2}{n^2}$, $r =
0$ and $B = \frac{C}{n}$ we arrive to the following inequality:
\begin{equation*}
\bpsi_i \leq C \frac{i^2}{n^2} q(n, \delta) e^{C \frac{i}{n}} \leq C'
\frac{i^2}{n^2} q(n,\delta).
\end{equation*}
\end{proof}

We now present an analogous result for the solution $\beta$ of
\eqref{eq:volterra}.Let $C_S = c_s \coeff$ and define the functions
\begin{equation*}
  \FF_s(v) = C_s F^{N-s}(v) \quad \GG_s(v) = G(v)^{s-1}.
\end{equation*}
It is not hard to verify that $z_s(v) = \FF_s(v) \GG_s(v)$ and that
the integral equation \eqref{eq:volterra} is given by
\begin{equation}
\label{eq:volterrabrief}
 \sum_{s=1}^S \int_0^v t (\FF_s(t) \GG_s(t))' dt = \sum_{s=1}^S
 \GG_s(v) \int_0^v \beta(t) \FF_s'(t)dt
\end{equation}
After differentiating this equation  and rearranging terms we obtain
\begin{align*}
 0 & = (v - \beta(v)) \sum_{s=1}^S \GG_s(v) \FF'_s(v)  \\
& + \sum_{s=1}^S  \GG_s'(v) \int_0^v \beta(t) \FF'_s(t) dt
+ v \GG_s'(v) \FF_s(v) \\
& = (v - \beta(v)) \sum_{s=1}^S \GG_s(v) \FF'_s(v)  \\
& + \sum_{s=1}^S \GG_s'(v) \int_0^v (t - \beta(t)) \FF_s'(t) dt 
+\GG'_s(v) \int_0^v \FF_s(t) dt,
\end{align*}
where the last equality follows from integration by parts.  Notice
that the above equation is the continuous equivalent of equation
\eqref{eq:finalorig}. Letting $\psi(v) :=  v - \beta(v)$ we have that
\begin{equation}
\label{eq:volterrapsi}
\psi(v)  = - \frac{ \sum_{s=1}^S \GG'_s(v) \int_0^v \FF_s (t) dt  +
\GG_s'(v) \int_0^v \psi(t) \FF_s'(t) dt} {\sum_{s=1}^S \GG_s(v) \FF_s'(v)} 
\end{equation}
Since $\lim_{v \rightarrow 0} \GG_s (v) = \lim_{v \rightarrow 0}
\GG_s'(v) /f(v) = 1$ and $\lim_{v \rightarrow 0} \FF_s(v) = 0$, it is
not hard to see that   
\begin{flalign*}
& -\psi(0) & \\
& =  \lim_{v \rightarrow 0} \frac{ \sum_{s=1}^S \GG'_s(v)
\int_0^v \FF_s (t) dt  + \GG_s'(v) \int_0^v \psi(t) \FF_s'(t) dt} {\sum_{s=1}^S \GG_s(v)
\FF_s'(v)} \\
& =  \lim_{v \rightarrow 0} \frac{ f(v) \Big(\sum_{s=1}^S \frac{\GG'_s(v)}{f(v)}
\int_0^v \FF_s (t) dt  \!+\! \frac{\GG_s'(v)}{f(v)} \int_0^v \psi(t)
\FF_s'(t) dt\Big)} {\sum_{s=1}^S \GG_s(v) \FF_s'(v)} \\
& =  \lim_{v \rightarrow 0} \frac{ f(v) \Big(\sum_{s=1}^S \int_0^v \FF_s(t) dt 
 + \int_0^v \psi(t) \FF_s'(t) dt \Big)} {\sum_{s=1}^S \FF_s'(v)} \\
\end{flalign*}
Since the smallest power in the definition of $\FF_s$ is attained at $s
= S$, the previous limit is in fact equal to:
\begin{align*}
& \lim_{v \rightarrow 0} \frac{ f(v) \Big( \int_0^v \FF_S(t) dt 
 + \int_0^v \psi(t) \FF_S'(t) dt \Big)} {\FF_S'(v)} \\
= &  \lim_{v \rightarrow 0} \frac{\int_0^v F^{N-S}(t) dt 
 + \int_0^v (N \!-\! S)\psi(t) F^{N-S-1}(t) f(t) dt} {(N\!-\!S)
 F^{N-S-1}(v)} \\
\end{align*}
Using L'Hopital's rule and simplifying we arrive to the
following:
\begin{equation*}
  \psi(0) = - \lim_{v \rightarrow 0} \frac{F^2(v)}{(N-S) (N -S - 1)
    f(v)} + \frac{\psi(v) F(v)}{(N - S - 1)} 
\end{equation*}
Moreover, since $\psi$ is a continuous function, it must be bounded
and therefore, the previous limit is equal to $0$. Using the same series of steps
we also see that:
\begin{flalign*}
& - \psi'(0) & \\
& = \lim_{v \rightarrow 0} \frac{\psi(v)}{v} \\
& = \lim_{v \rightarrow 0} \frac{ \int_0^v F^{N-S}(t) dt 
 + \int_0^v (N \!-\! S)\psi(t) F^{N-S-1}(t) f(t) dt } {v (N-S) F^{N-S-1}(v)}
\end{flalign*}
By L'Hopital's rule again we have the previous limit is equal to
\begin{equation}
\label{eq:lhopital}
\lim_{v \rightarrow 0} \frac{ F^{N-S}(v) + (N-S) \psi(v) F^{N-S-1}(v)
  f(v)} {(N\!\!-\!\!S) (N\!\! -\!\! S \!\!-\!\! 1) F^{N\!\!-\!\!S\!-\!2}(v) f(v) v  + (N\!\!-\!\!S) F^{N-S-1} (v)}
\end{equation}
Furthermore, notice that 
\begin{align*}
&\lim_{v \rightarrow 0} \frac{ F^{N-S}(v) + (N-S) \psi(v) F^{N-S-1}(v)
  f(v)} {(N-S) (N - S - 1) F^{N-S-2}(v) f(v) v}
\\
= & \lim_{v \rightarrow 0} \frac{ F^2(v)}{(N-S)(N-S -1) f(v) v} 
+  \frac{\psi(v) F(v)}{(N-S - 1) v}  = 0.
\end{align*}
Where for the last equality we used the fact that $\lim_{v \to 0}
\frac{F(v)}{v} = f(0)$  and $\psi(0) = 0$. Similarly, we have:
\begin{multline*}
\lim_{v \rightarrow 0} \frac{ F^{N-S}(v) + (N-S) \psi(v)
  F^{N-S-1}(v) f(v) }{(N-S)F^{N-S-1}(v)}
\\ = \lim_{v \to 0} \frac{F(v)}{N-S} + \psi(v) f(v) = 0
\end{multline*}
Since the terms in the denominator of \eqref{eq:lhopital} are
positive, the two previous limits imply that the limit given by
\eqref{eq:lhopital} is in fact $0$ and therefore $\psi'(0) = 0$. 
Thus, by  Taylor's theorem we have $|\psi(v)|
\leq C v^2$ for some constant $C$.
\begin{corollary}
\label{coro:smallerr}
The following inequality holds with probability at least $1 - \delta$
for all $i \leq \frac{1}{n^{3/4}}$
\begin{equation*}
  | \bpsi_i - \psi(v_i) | \leq C \frac{1}{\sqrt{n}} q(n,\delta).
\end{equation*}
\begin{proof}
Follows trivially from the bound on $\psi(v)$,
Proposition~\ref{prop:smallemp} and the fact that $\frac{i^2}{n^2} \leq
\frac{1}{\sqrt{n}}$. 
\end{proof}
\end{corollary} 
Having bounded the magnitude of the error for small values of $i$ one
could use the forward substitution technique used in
Proposition~\ref{prop:smallemp} to bound the errors $\e_i =
| \bpsi_i - \psi(v_i)|$. Nevertheless, a crucial assumption used in
Proposition~\ref{prop:smallemp} was the fact that $\bpsi_i \geq
0$. This condition, however is not necessarily verified by
$\e_i$. Therefore, a forward substitution technique will not
work. Instead, we leverage the fact that $|d \M_{i,i-1} \bpsi_{i-1} -
d \M_{i,i} \bpsi_i|$ is in $O\big(\frac{1}{n^2} \big)$ and show that
the solution $\overline{\bpsi}$ of a surrogate linear equation is
close to both $\bpsi$ and $\psi$ implying that $\bpsi_i$ and
$\psi(v_i) $ will be close too.  Therefore let $d \M'$ denote the
lower triangular matrix with $d \M'_{i,j} = d \M_{i,j}$ for $j \leq
{i-2}$, $d \M'_{i,i-1} = 0$ and $d \M'_{ii} = 2 d \M_{ii}$. Thus, we
are effectively removing the problematic term $d \M_{i,i-1}$ in the
analysis made by forward substitution. The following proposition
quantifies the effect of approximating the original system with the
new matrix $d \M'$.

\begin{proposition}
  \label{prop:surrogate}
Let $\overline \bpsi$ be the solution to the system of equations
\begin{equation*}
  d \M' \overline \bpsi = \mat p.
\end{equation*}
Then, for all $i \in \{1, \ldots, n\}$ it is true that
\begin{equation*} 
|\bpsi_i - \overline \bpsi| \leq \Big (\frac{1}{\sqrt{n}} +
\frac{q(n,\delta)}{n^{3/2}}\Big)e^{C}.
\end{equation*}
\end{proposition}
\begin{proof}
 We can show, in the same way as in Proposition~\ref{prop:smallemp}, that
 $\overline{\bpsi_i} \leq C \frac{i^2}{n^2} q(n, \delta)$ with
 probability at least $1 - \delta$ for all $i$. In particular, for $i
 < n^{3/4}$ it is true that
 \begin{equation*}
|\bpsi_i - \overline \bpsi_i| \leq C \frac{1}{\sqrt{n}} q(n, \delta). 
 \end{equation*}
On the other hand by forward substitution we have
\begin{equation*}
 d \M'_{ii} \overline{\bpsi}_i = \mat p_i - \sum_{j=1}^{i-1} d \M'_{ij}
 \overline{\bpsi}_j
\end{equation*}
and
\begin{equation*}
d \M_{ii} \bpsi_i = \mat p_i - \sum_{j=1}^{i-1} d \M_{ij} \bpsi_j.
\end{equation*}
By using the definition of $d \M'$ we see the above equations hold if
and only if
\begin{align*}
 2 d \M_{ii} \overline{\bpsi}_i &= \mat p_i - \sum_{j=1}^{i-2} d \M_{ij}
 \overline{\bpsi}_j \\
2 d \M_{ii} \bpsi_i &=  d \M_{ii} \bpsi_i + \mat p_i -  d \M_{i,i-1} - \sum_{j=1}^{i-2} d \M_{ij} \bpsi_j.
\end{align*}
Taking the difference of these two equations yields a recurrence
relation for the quantity $e_i = \bpsi_i - \overline{\bpsi}_i$.
\begin{equation*}
 2 d \M_{ii} e_i =  d \M_{ii} \bpsi_i  - d \M_{i,i-1} \bpsi_{i-1} - \sum_{j=1}^{i-2} d
 \M_{ij} e_j. 
\end{equation*}
Furthermore we can bound  $d \M_{ii} \bpsi_i  - d \M_{i,i-1} \bpsi_{i-1}$ as
follows:
\begin{flalign*}
&|d \M_{ii} \bpsi_i  - d \M_{i,i-1}\bpsi_{i-1}| & \\
& \leq | d \M_{ii} - d \M_{i,i-1}| \bpsi_{i-1}
+ |\bpsi_i -  \bpsi_{i-1}| d \M_{ii}. \\
& \leq C \frac{i^{p}}{n^{p}}\frac{q(n,\delta)}{n^2} + \frac{C}{\sqrt{n}} d \M_{ii}.
\end{flalign*}
Where the last inequality follows from Assumption~\ref{assum:smooth}
and Proposition~\ref{prop:subdiag} as well as from the fact that
$\bpsi_i \leq \frac{i^2}{n^2} q(n, \delta)$. Finally, using the same bound  on
$d \M_{ij}$ as in Proposition~\ref{prop:smallemp} gives us
\begin{align*}
|e_i|
& \leq C \Big( \frac{q(n,\delta)}{n i} + \frac{1}{\sqrt{n}} + \frac{1}{n}
\sum_{j=1}^{i-2} e_i \Big) \\
& \leq C \frac{1}{\sqrt{n}} + \frac{C}{n} \sum_{j=1}^{i-2} e_i.
\end{align*}
Applying Lemma~\ref{lemma:gromwall} with $A = \frac{C}{\sqrt{n}}$, $B =
\frac{C}{n}$ and $r = n^{3/4}$ we obtain the final bound
\begin{equation*}
|\bpsi_i - \overline{\bpsi}_i| \leq \Big (\frac{1}{\sqrt{n}} +
\frac{q(n,\delta)}{n^{3/2}}\Big)e^{C}.
\end{equation*}
\end{proof}

\section{PROOF OF THEOREM~\ref{th:convergence}}
\label{sec:convergence-proof}

\begin{proposition}
\label{prop:errors}
Let $\psi(v)$ denote the solution of \eqref{eq:volterrapsi} and denote
by $\h \bpsi$ the vector defined by $\h \bpsi_i = \psi(v_i)$. Then, with
probability at least $1 - \delta$
\begin{equation}
\label{eq:errors}
 \max_{i > \sqrt{n}}  n |(d \M' \h \bpsi)_i - \mat p_i| \leq C
 \frac{i^{N-S}}{n^{N-S}} \frac{\log(2/\delta)^{N/2}}{\sqrt{n}} 
q(n, \delta/2)^3.
\end{equation}
\end{proposition}
\begin{proof}
By definition of  $ d \M' $ and $\mat p_i$ we can
decompose the difference $n\big( (d \M' \h \bpsi)_i- \mat p_i\big)$ as:
\begin{multline}
\label{eq:decomp}
\sum_{s=1}^S c_s \Bigg( I_s(v_i) + \Upsilon_3(v_i) -  \big(
\Upsilon_1(s, i) + \Upsilon_2(s, i) \big) \\
- n\Delta \mat v_i \coeff \F_{i-1}^{p} \Big( \G_{i-1}^{s-1}
+ \frac{ n \Delta \G_i^s }{s} \Big) \Bigg). 
\end{multline}
where
\begin{flalign*}
&\Upsilon_1(s, i) & \\
& = \coeff \frac{n^2 \Delta^2 \G_i^s}{s} \sum_{j=1}^{i-1} \F_{j-1}^{p} \Delta \mat
v_j - \frac{\GG_s'(v_i)}{f(v_i)} \int_0^{v_i} \FF_s(t),\\
&\Upsilon_2(s, i) = \coeff \frac{n^2 \Delta^2 \G_i^s}{s} \sum_{j=1}^{i-2} \Delta \F^{p}_i
\psi(v_j)\\
& \mspace{80mu} - \frac{\GG_s'(v_i)}{f(v_i)} \! \int_0^{v_i} \FF_s'(t)
\psi(t) dt, \\
&\Upsilon_3(s,i) = \Big(2 n \M_{ii}(s)  - \frac{\FF_s'(v_i)}{f(v_i)} \GG_s(v_i)
\Big) \psi(v_i) 
\qquad   \text{and} \\
&I_s(v_i) = \frac{1}{f(v_i)} \Big( \FF_s'(v_i)\GG_s(v_i) \psi(v_i) +\GG_s'(v_i)
\int_0^{v_i} \FF_s(t) \\
&\mspace{80mu} + \GG_s'(v_i) \! \int_0^{v_i} \FF_s'(t) \psi(t) dt\Big).
\end{flalign*}
Using the fact that $\psi$ solves equation \eqref{eq:volterrapsi} we
see that $\sum_{s=1}^S c_s I_s(v_i) = 0$.  Furthermore, using
Lemma~\ref{lemma:derivative} as well as Proposition~\ref{prop:hp} we have
\begin{align*}
 n\Delta \mat v_i \coeff \F_{i-1}^{p} \Big( \G_{i-1}^{s-1}+ \frac{ n
 \Delta \G_i^s }{s} \Big)
& \leq \frac{i^p}{n^p}\frac{1}{n} q(n,\delta/2) \\
& \leq \frac{i^{N-S}}{n^{N-S}} \frac{1}{n} q(n, \delta/2)
\end{align*}
 Therefore we need only to bound $\Upsilon_k$ for $k =1,2,3$. After
replacing the values of $\GG_s$ and $\FF_s$ by its definitions,
Proposition~\ref{prop:Miiapprox} and the fact that $\psi(v_i) \leq C
v_i^2 \leq C \frac{i^2}{n^2} q^2(n,\delta)$ imply that with
probability at least $1 - \delta$
\begin{equation*}
  \Upsilon_3(s, v_i) \leq C \frac{i^p}{n^p}
  \frac{\log(2/\delta)^{p-2}}{\sqrt{n}} q(n, \delta/2)^3.
\end{equation*}
We proceed to bound the term $\Upsilon_2$. The bound for $\Upsilon_1$
can be derived in a similar manner. By using the definition of $\GG_s$
and $\FF_s$ we see that $\Upsilon_2 = \coeff \big( \Upsilon_2^{(1)} +
\Upsilon_2^{(2) } \big)$ where 
\begin{multline*}
\begin{aligned}
& \Upsilon_2^{(1)}(s,i)  
= \Big(\frac{n^2 \Delta^2 \G_i^s}{s} -
(s-1)G(v_i)^{s-2} \Big) \sum_{j=1}^{i-2} \Delta \F_i^p \h \bpsi_i \\
& \Upsilon_2^{(2)}(s, i) 
\end{aligned} \\
= \Big( \sum_{j=1}^{i-2} \Delta \F_i^p \h\bpsi 
- \int_0^{v_i} \psi(t) p F^{p-1}(t)f(t)dt\Big) (s\!-\!1) G(v_i)^{s-2}.
\end{multline*}
It follows from Propositions~\ref{prop:intbounds} and
\ref{prop:deltadiff} that $|\Upsilon_2(s,i)| \leq C \frac{i^p}{n^p} \frac{
\log(2/\delta)^{p/2}}{\sqrt{n}} q(n, \delta/2)^2$. And the same
inequality holds for $\Upsilon_1$. Replacing these bounds in
\eqref{eq:decomp} and using the fact $\frac{i^p}{n^p} \leq
\frac{i^{N-S}}{n^{N-S}}$ yields the desired inequality. 
\end{proof}

\begin{figure}[t]
\vskip -.3in
\centering
\begin{tabular}{c}
(a)\includegraphics[scale=.39]{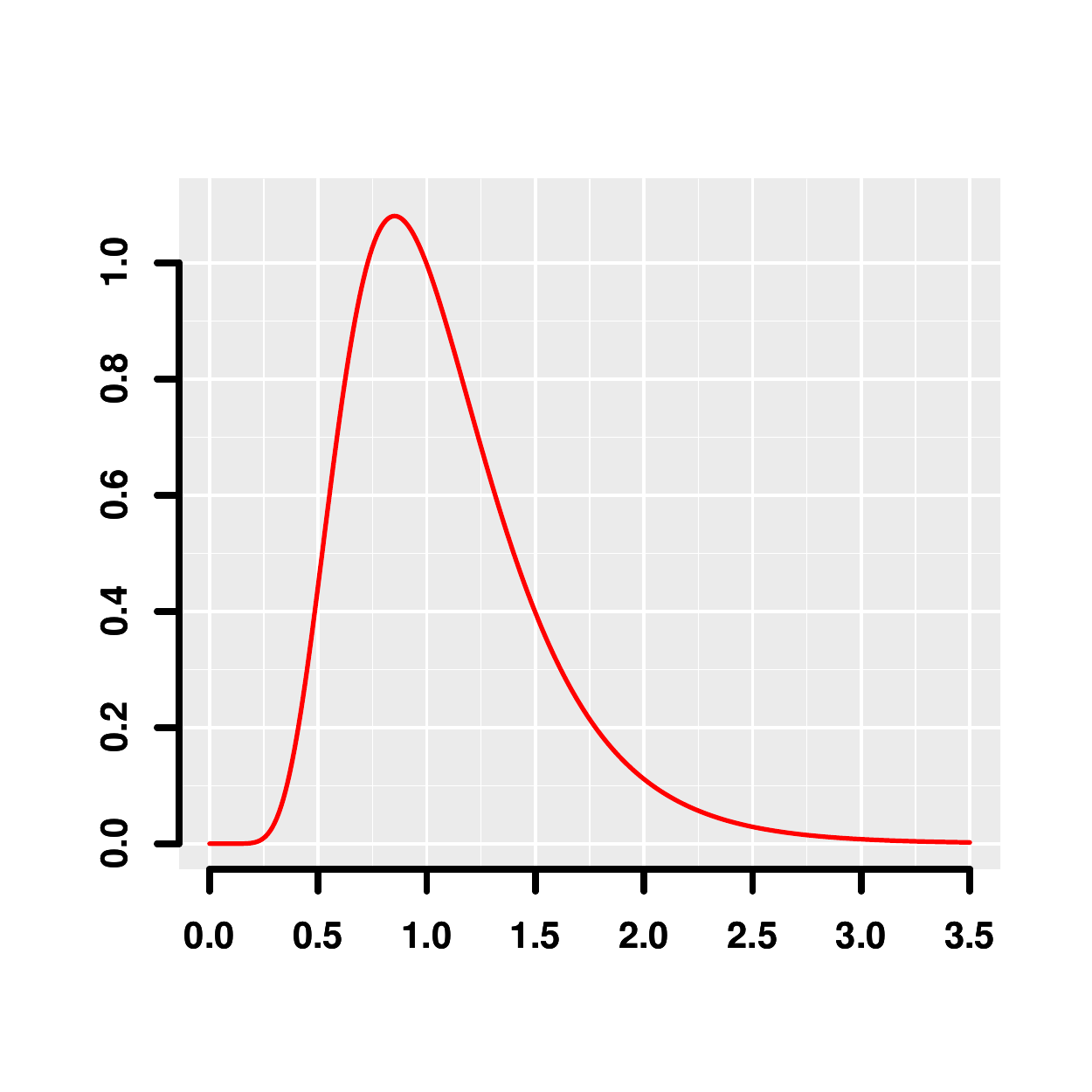} \\
(b)\includegraphics[scale=.39]{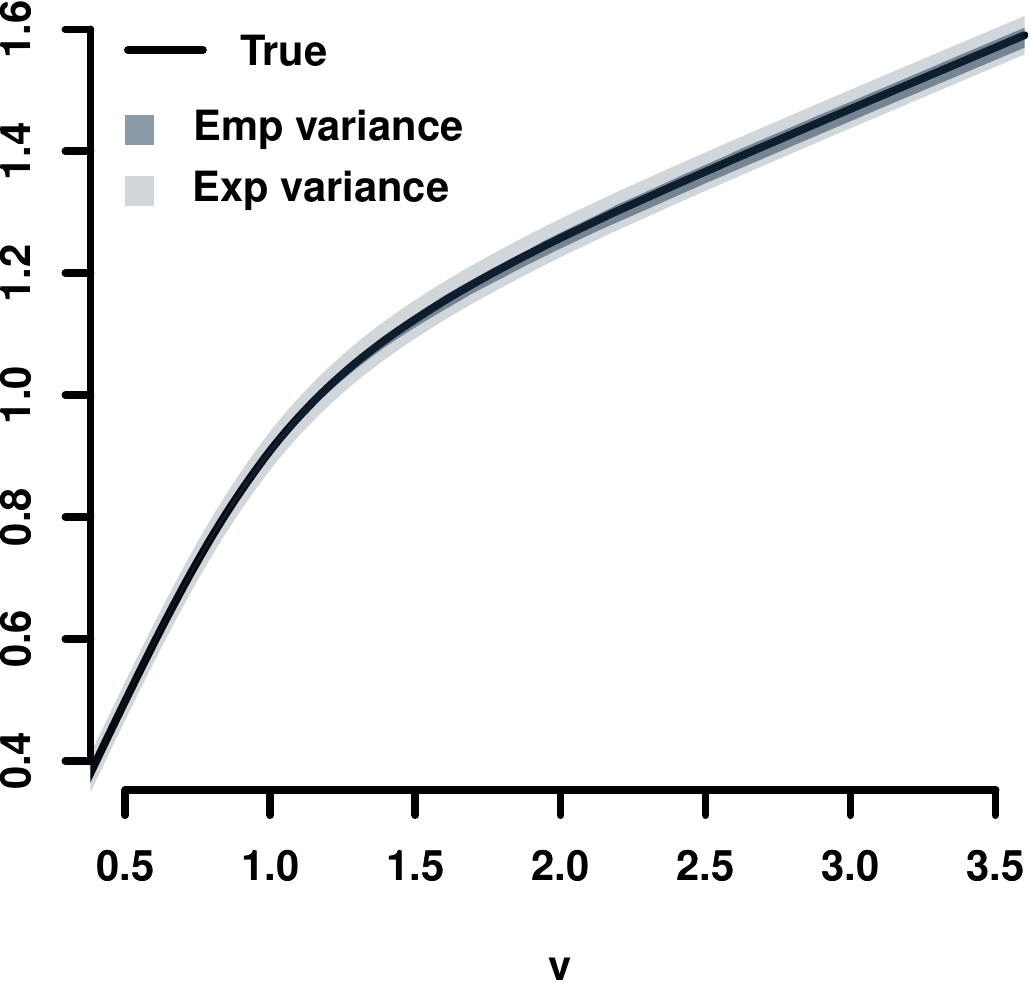} \\
(c)\includegraphics[scale=.39]{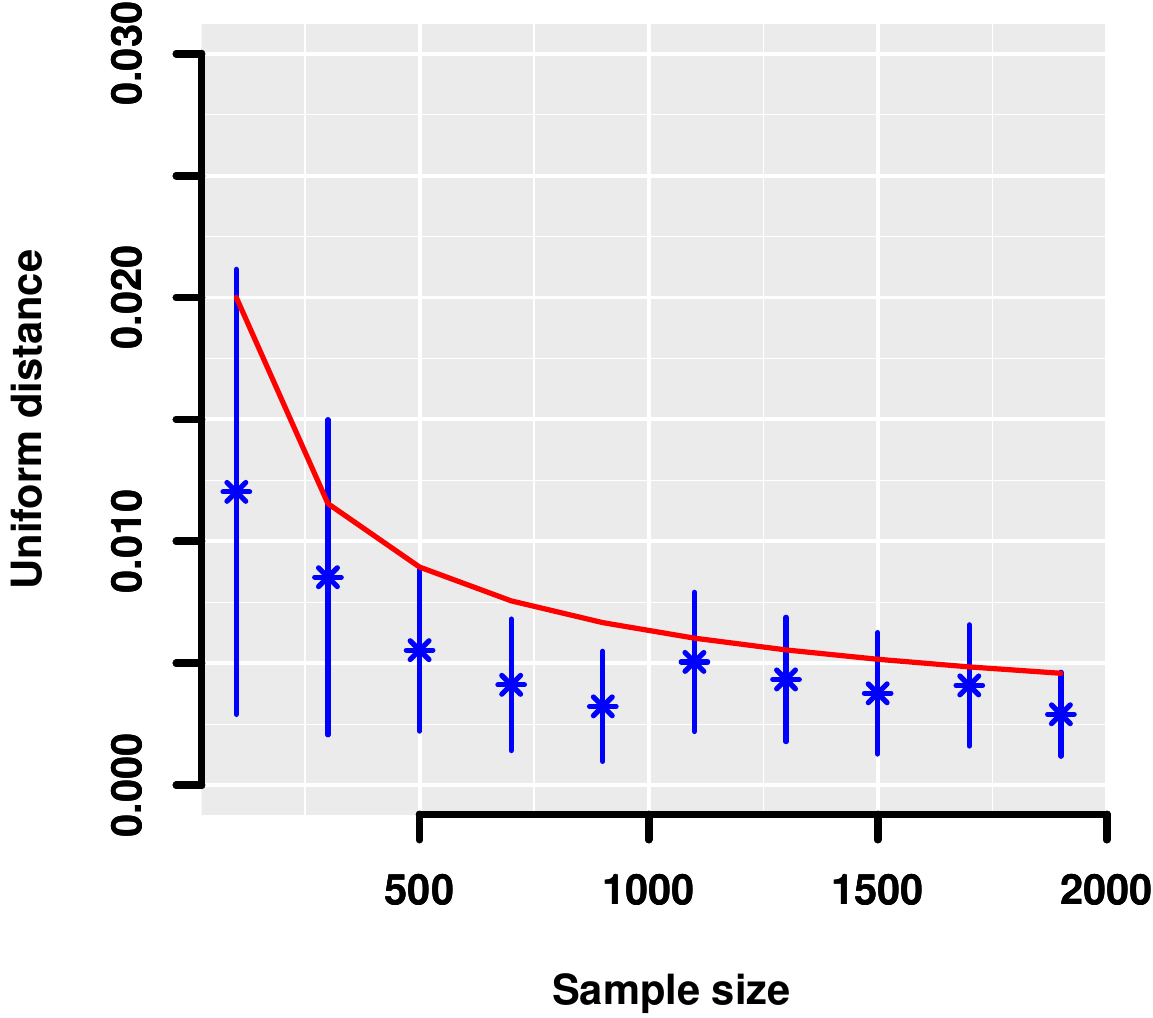}
\end{tabular}
\caption{(a) Log-normal density used to sample valuations. (b) True
  equilibrium bidding function and empirical approximations (in dark
  grey) and theoretical $\frac{1}{\sqrt{n}}$ confidence bound around
  true bidding function. (c) Rate of convergence to equilibrium as a
  function of the sample size, the red line represents the function
  $0.2/\sqrt(n)$.}
\label{fig:lndens}
\end{figure} 

\begin{proposition}
For any $\delta > 0$, with probability at least $1 - \delta$ 
\begin{multline*}
  \max_{i} |\psi(v_i) - \overline \bpsi_i|\\
\leq  e^C\Big(\frac{\log(2/\delta)^{N/2}}{\sqrt{n}} 
q(n, \delta/2)^3 + \frac{C q(n, \delta/2)}{n^{3/2}}\Big)
\end{multline*}
\end{proposition}
\begin{proof}
With the same argument used in Corollary~\ref{coro:smallerr} we see that
with probability at least $1 - \delta$ for $i \leq
\frac{1}{n^{3/4}}$ we have $|\psi(v_i) - \overline \bpsi_i| \leq
\frac{C}{\sqrt{n}} q(n, \delta)$. On the other hand, since $d \M_i =
\mat p_i$ the previous Proposition implies that for $i > n^{3/4}$
\begin{equation*}
  n \big| \big(d \M ' (\h \bpsi - \overline \bpsi) \big)_i  \big| \leq  C
 \frac{i^{N-S}}{n^{N-S}} \frac{\log(2/\delta)^{N/2}}{\sqrt{n}} 
q(n, \delta/2)^3.
\end{equation*}
Letting $\e_i  = |\psi(v_i) - \overline \bpsi_i|$, we see that the
previous equation defines the following recursive inequality.
\begin{multline*}
n d \M'_{ii} \e_i 
\leq  C  \frac{i^{N-S}}{n^{N-S}} \frac{\log(2/\delta)^{N/2}}{\sqrt{n}} 
q(n, \delta/2)^3 \\ 
- C n\sum_{j=1}^{i-2} d \M'_{ij} \e_j,
\end{multline*}
where we used the fact that $d \M'_{i,i-1} = 0$. Since $d \M'_{ii} =
2 \M_{ii} \geq 2 \overline C \frac{i^{N-S-1}}{n^{N-S-1}} \frac{1}{n}$, after
dividing the above inequality by $d \M'_{ii}$ we obtain
\begin{equation*}
 \e_i \leq  C \frac{\log(2/\delta)^{N/2}}{\sqrt{n}} 
q(n, \delta/2)^3 - \frac{C}{n} \sum_{j=1}^{i-2} \e_j.
\end{equation*}
Using Lemma~\ref{lemma:gromwall} again we conclude that 
\begin{equation*}
  \e_i \leq e^C \Big(\frac{\log(2/\delta)^{N/2}}{\sqrt{n}} 
q(n, \delta/2)^3 + \frac{C q(n, \delta/2)}{n^{3/2}}\Big) 
\end{equation*}
\end{proof}
\begin{reptheorem}{th:convergence}
If Assumptions~\ref{assum:dens}, \ref{assum:smooth} and
\ref{assum:smoothreal} are satisfied, then, for any $\delta > 0$, with
probability at least $1 - \delta$ over the draw of a sample of size
$n$, the following bound holds for all $i \in [1, n]$:
\begin{equation*}
| \h \beta(v_i) - \beta(v_i) | \leq
e^C\Big(\frac{\log(2/\delta)^{N/2}}{\sqrt{n}} q(n, \delta/2)^3 +
\frac{C q(n, \delta/2)}{n^{3/2}}\Big).
\end{equation*}
where $q(n, \delta) = \frac{2}{c}\log(nc/2\delta)$ with $c$ defined in
Assumption~\ref{assum:dens}, and where $C$ is some universal constant.
\end{reptheorem}
\begin{proof}
The proof is a direct consequence of the previous proposition and
Proposition~\ref{prop:surrogate}.
\end{proof}

\section{EMPIRICAL CONVERGENCE}
Here we present an example of convergence by the empirical bidding
functions to the true equilibrium bidding function, even when not all
technical assumptions are verified. We sampled valuations from a
log-normal distribution of parameters $\mu = 0$ and $\sigma = 0.4$ and
calculated the empirical bidding function. Notice that in this case,
the support of the distribution is not bounded away from zero (see
Figure~\ref{fig:lndens}(a). Figure~\ref{fig:lndens}(b) shows the true
equilibrium bidding function as well as the range of empirical
equilibrium functions (in dark grey) obtained after repeating this
experiment 10 times. Finally, the region in light gray depicts the
predicted theoretical confidence bound in
$O(\frac{1}{\sqrt{n}}$. Figure~\ref{fig:lndens}(c) shows the rate of
uniform convergence to the true equilibrium function as a function of
$n$.
\end{document}